\documentclass[11pt]{article}
\usepackage[margin=1in]{geometry}
\usepackage{fullpage}
\usepackage{natbib}
\usepackage{xspace}
\usepackage{authblk}
\usepackage{times}
\usepackage{algorithm, algorithmic}
\usepackage{mkolar_definitions}
\usepackage{dsfont}
\usepackage[breaklinks=true]{hyperref}
\usepackage{breakcites}

\title{Improved Algorithms for Efficient Active Learning Halfspaces with Massart and Tsybakov Noise}

\author[]{Chicheng Zhang\thanks{Email: chichengz@cs.arizona.edu} }
\author[]{Yinan Li\thanks{Email: yinanli@email.arizona.edu}}
\affil[]{University of Arizona}

\newcommand{\initialize}{\ensuremath{\textsc{Initialize}}\xspace}
\newcommand{\refine}{\ensuremath{\textsc{Refine}}\xspace}
\newcommand{\optimize}{\ensuremath{\textsc{Optimize}}\xspace}
\DeclareMathOperator{\sign}{sign}

\DeclareMathOperator{\err}{err}
\DeclareMathOperator{\polylog}{polylog}
\DeclareMathOperator{\poly}{poly}
\newcommand{\EX}{\mathrm{EX}}
\newcommand{\opt}{\mathrm{opt}}
\renewcommand{\ind}{\mathds{1}}
\newcommand{\pot}{\psi_{D,b}}
\newcommand{\agg}{\mathrm{agg}}
\newcommand{\avg}{\mathsf{average}}
\newcommand{\rnd}{\mathsf{random}}
\newcommand{\GTNC}{\mathsf{GTNC}}
\newcommand{\TNC}{\mathsf{TNC}}
\newcommand{\MNC}{\mathsf{MNC}}
\newcommand{\unif}{\mathrm{unif}}
\newcommand{\disag}{R}
\newcommand{\vecz}{\overrightarrow{0}}
\newcommand{\hide}[1]{}
\newcommand{\HT}{\mathrm{HT}}

\def\shownotes{1}  \ifnum\shownotes=1
\newcommand{\authnote}[2]{$\ll$\textsf{\footnotesize #1 notes: #2}$\gg$}
\else
 \newcommand{\authnote}[2]{}
\fi

\begin{document}

\maketitle






\begin{abstract}
We give a computationally-efficient PAC active learning algorithm for $d$-dimensional homogeneous halfspaces that can tolerate Massart noise~\citep{massart2006risk} and Tsybakov noise~\citep{tsybakov2004optimal}. Specialized to the $\eta$-Massart noise setting, our algorithm achieves an information-theoretically near-optimal label complexity of $\tilde{O}\rbr{\frac{d}{(1-2\eta)^2} \polylog(\frac1\epsilon)}$ under a wide range of unlabeled data distributions (specifically, the family of ``structured distributions'' defined in~\citet{diakonikolas2020polynomial}). Under the more challenging Tsybakov noise condition, we identify two subfamilies of noise conditions, under which our efficient algorithm  
provides label complexity guarantees strictly lower than passive learning algorithms.
\end{abstract}



  
  


\section{Introduction}
Motivated by the abundance of unlabeled data and the expensiveness of obtaining labels, the paradigm of active learning has been proposed and extensively studied in the literature~\cite[see e.g.][for  comprehensive surveys]{settles2009active,hanneke2014theory}.
In active learning, a learner starts with a set of unlabeled examples, 
and can adaptively select subsets of them to query for their labels. Thanks to its adaptivity, an active learner can focus on obtaining informative labels, and can thus substantially reduce labeling effort compared to conventional supervised learning.

Halfspaces, also known as linear separators, are arguably one of the most fundamental concept classes studied in machine learning and data analysis.  
Significant research efforts on halfspace learning from computational and statistical perspectives have resulted in rich theory~\cite[e.g.][]{vapnik1998statistical,blum1996polynomial,zhang2004statistical,bartlett2006convexity, kalai2008agnostically} and many practical algorithms~\citep[e.g.][]{cortes1995support, cristianini2000introduction}.

Label noise is ubiquitous in machine learning applications due to various factors, such as human error, sensor failure, etc~\citep{balcan2020noise}, and it is therefore important to design learning algorithms that are robust to label noise. 
If computational efficiency is not of  concern, classical methods such as empirical risk minimization are known to achieve statistical consistency~\citep{vapnik1998statistical}. 
However, in many practical applications, it is often necessary for learners to process its training examples in a computationally efficient manner. 
Therefore, it is of importance to develop computationally efficient, noise-tolerant learning algorithms with statistical consistency guarantees.

However, it is now well-understood that without additional assumptions on the label noise, agnostically learning halfspaces is computationally hard~\citep{feldman2006new,guruswami2009hardness,daniely2016complexity}, even under well-behaved unlabeled data distributions such as standard Gaussian~\citep{klivans2014embedding,diakonikolas2020near}. This motivates the study of learning halfspaces under more benign label noise conditions. Massart noise~\citep{massart2006risk} and Tsybakov noise~\citep{tsybakov2004optimal} are two noise models widely studied in the literature; specialized to the halfspace learning setting, they are formally defined as:

\begin{definition}[Massart noise condition]
Given $\eta \in [0,\frac12)$,
a distribution $D$ over $\RR^d \times \cbr{-1,+1}$ is said to satisfy the {\em $\eta$-Massart noise condition} with respect to halfspace $w^\star \in \RR^d$, if for all examples $x$, $\eta(x) \leq \eta$, where $\eta(x) = \PP_D(y \neq \sign(\inner{w^\star}{x}) \mid x)$.
\label{def:mnc}
\end{definition}

\begin{definition}[Tsybakov noise condition]
Given $A > 0$ and $\alpha \in (0,1]$, a distribution $D$ over $\RR^d \times \cbr{-1,+1}$ is said to satisfy the {\em $(A, \alpha)$-Tsybakov noise condition} with respect to halfspace $w^\star \in \RR^d$, if for all $t \in [0,\frac12)$,  $\PP_D\rbr{ \frac12 - \eta(x) \leq t} \leq A t^{\frac{\alpha}{1-\alpha}}$,
where $\eta(x) = \PP_D(y \neq \sign(\inner{w^\star}{x}) \mid x)$.
\label{def:tnc}
\end{definition}

Although nearly-matching upper and lower bounds on sample and label complexities have been established in both supervised (passive) and active learning settings under these two noise conditions~\citep[e.g.][]{massart2006risk, tsybakov2004optimal, castro2008minimax, hanneke2011rates, balcan2013active, wang2016noise}, most of these results are only {\em statistical}: the algorithms that achieve the sample or label complexity upper bounds are computationally inefficient. Only recently have computationally efficient algorithms been proposed in the literature~\citep[e.g.][]{awasthi2015efficient, diakonikolas2019distribution, diakonikolas2020polynomial} under these noise conditions; see Section~\ref{sec:relwork} for detailed discussions. Still, these works leave out two important open questions:
\begin{enumerate}
    \item {\sloppy  Are there efficient active halfspace learning algorithms that tolerate Massart noise with near-optimal label complexity, under a broad range of unlabeled data distributions? Specifically, the algorithm of~\cite{yan2017revisiting}
    achieves an information-theoretically near-optimal label complexity of $O\rbr{\frac{d}{(1-2\eta)^2} \polylog(\frac1\epsilon)}$, but relies on the strong assumption that the unlabeled data distribution is uniform on the unit sphere; under broader distributional assumptions such as log-concave distributions,
    the state-of-the-art algorithm of~\cite{zhang2020efficient} only achieves a suboptimal label complexity of $O\rbr{\frac{d}{(1-2\eta)^4} \polylog(\frac1\epsilon)}$. Can we design efficient algorithms with $\tilde{O}\rbr{\frac{d}{(1-2\eta)^2} \polylog(\frac1\epsilon)}$ label complexity under broader unlabeled data distributions, for example, the family of isotropic log-concave distributions~\citep{lovasz2007geometry}? }

    
    \item Are there efficient active halfspace learning algorithms that tolerate Tsybakov noise with label complexities better than passive learning? The state-of-the-art work of~\cite{diakonikolas2020polynomial} propose efficient algorithms for passive learning halfspaces with Tsybakov noise, with sample complexities  $O\rbr{(\frac{d}{\epsilon})^{O(\frac{1}{\alpha})}}$ and $O\rbr{ \poly(d) (\frac 1 \epsilon)^{O(\frac1{\alpha^2})} }$; moreover, computationally inefficient algorithms such as empirical risk minimization achieve a sharper sample complexity of $O\rbr{  d(\frac{1}{\epsilon})^{2-\alpha}}$~\cite[e.g.][Chapter 3]{hanneke2014theory}.
    Can we design efficient active learning algorithms with label complexities of strictly lower order than these? A positive answer to this question can serve as a stepping stone towards developing efficient active algorithms with label complexity matching those of computationally inefficient active learning algorithms~\citep[e.g.][]{balcan2013active}, which is $O\rbr{ d(\frac 1 \epsilon)^{2-2\alpha}}$.  
\end{enumerate}

\paragraph{Our results.} Our work answers the above two questions in the affirmative. Specifically, under a set of structural assumptions on the unlabeled data distribution~\citep{diakonikolas2020polynomial} (see also Definition~\ref{def:well-behaved} in Section~\ref{sec:setting}), we give an efficient PAC active halfspace learning algorithm, such that  with appropriate settings of its parameters:
\begin{enumerate}
    \item under the $\eta$-Massart noise condition, it has  an information-theoretically near-optimal label complexity of $O\rbr{\frac{d}{(1-2\eta)^2} \polylog(\frac1\epsilon)}$. This substantially weakens the distributional requirements to achieve such near-optimal label complexity results; before our work, such result is only known when the unlabeled data distribution is  uniform over the unit sphere~\citep{yan2017revisiting}, or when the noise parameter $\eta$ is $3 \times 10^{-6}$, a tiny constant~\citep{awasthi2015efficient}. Furthermore, when the unlabeled data distribution is isotropic log-concave, our result improves over the recent work of of~\cite{zhang2020efficient}, where an efficient algorithm with suboptimal label complexity $O\rbr{\frac{d}{(1-2\eta)^4} \polylog(\frac1\epsilon)}$ is proposed. 
    
    \item under the $(A, \alpha)$-Tsybakov noise condition with $\alpha \in (\frac12, 1]$, it has a label complexity of $\tilde{O}(d (\frac{1}{\epsilon})^{\frac{2-2\alpha}{2\alpha-1}})$. 
    Specifically, when $\alpha > \frac{7-\sqrt{17}}{4} \approx 0.719$, $\frac{2-2\alpha}{2\alpha-1} < 2-\alpha$, in which case our algorithm has a better label complexity than passive learning.
    Furthermore, in the special case of $(B,\alpha)$-geometric Tsybakov noise condition (see Definition~\ref{def:gtnc} in Section~\ref{sec:setting}), our algorithm achieves a lower label complexity of $O\rbr{ d (\frac{1}{\epsilon})^{\frac{2}{\alpha} - 2} }$ for all $\alpha \in (0,1]$; specifically, when $\alpha > 2 - \sqrt{2} \approx 0.585$, $\frac{2-2\alpha}{\alpha} < 2 - \alpha$, in which case our algorithm has a better label complexity than passive learning.
\end{enumerate}

\paragraph{Techniques.} Our algorithm and analysis bear similarities to the recent work of~\cite{zhang2020efficient}, who observe that online mirror descent-style updates, when composed with a margin-based active sample selection rule~\citep{balcan2007margin}, implicitly minimizes a nonstandard proximity measure against the Bayes-optimal halfspace $w^\star$. 
Compared to~\citet{zhang2020efficient}, our results are novel in two aspects. 
First, the algorithm of~\cite{zhang2020efficient} is specialized to Massart noise, as its update rule crucially relies on the knowledge of the Massart noise level $\eta$. In contrast, we propose a new and simpler update rule that can work under both Massart and Tsybakov noise conditions.
Second, under the $\eta$-Massart noise setting, ~\cite{zhang2020efficient} use an averaging-based initialization procedure, which leads to an algorithm and analysis requiring a suboptimal label complexity of $\tilde{O}\rbr{ \frac{d}{(1-2\eta)^4} \polylog(\frac1\epsilon)}$. In this paper, we design a new initialization procedure with improved label efficiency, leading to an algorithm with an information-theoretically near-optimal label complexity of $\tilde{O}\rbr{ \frac{d}{(1-2\eta)^2} \polylog(\frac1\epsilon) }$ under Massart noise.

\section{Related work}
\label{sec:relwork}
    
    
    

\paragraph{Learning halfspaces under Massart and Tsybakov noise: statistical rates.} For passive learning, it is well-known that empirical risk minimization achieves minimax-optimal sample complexities of $\tilde{O}\rbr{\frac{d}{(1-2\eta) \epsilon}}$ and $\tilde{O}\rbr{d (\frac{1}{\epsilon})^{2-\alpha}}$ under $\eta$-Massart noise and $(A,\alpha)$-Tsybakov noise conditions respectively~\citep[e.g.][Chapter 3]{hanneke2014theory}.
For active learning, many works have provided  distribution-specific label complexity upper bounds, including the general analyses of~\cite{hanneke2011rates,beygelzimer2010agnostic,zhang2014beyond} and more specialized analyses of~\cite{balcan2013active,wang2016noise}.
In the setting of $\eta$-Massart noise,  under the assumption that the unlabeled data distribution is isotropic log-concave, the state-of-the-art algorithms of~\cite{balcan2013active, zhang2014beyond} have a label complexity of $\tilde{O}\rbr{ \frac{d}{(1-2\eta)^2} \polylog( \frac1\epsilon) }$; in the setting of $(A,\alpha)$-Tsybakov noise under log-concave unlabeled distributions, the state-of-the-art algorithms of~\cite{balcan2013active,zhang2014beyond,wang2016noise} achieve a label complexity of $\tilde{O}\rbr{ d (\frac{1}{\epsilon})^{2-2\alpha} }$.
Although these algorithms provide sharp label complexity guarantees, 
they all suffer from computational inefficiency: they need to perform empirical 0-1 loss minimization, which is known to be NP-hard in general~\citep{arora1997hardness}. 

\paragraph{Efficient passive learning halfspaces with Massart noise.}
The study of computationally efficient halfspace learning  under Massart noise is initiated by the work of~\cite{awasthi2015efficient}, who provide a PAC learning algorithm that works under the assumptions that the unlabeled distribution is isotropic log-concave, and the Massart noise parameter $\eta$ is smaller than a tiny constant ($3 \times 10^{-6}$). Prior to this work, positive results mainly focus on the much weaker random classification noise~\citep[e.g.][]{blum1996polynomial,balcan2013statistical}. Under similar distributional assumptions,~\citet{awasthi2016learning} propose an algorithm with sample complexity $O\rbr{ d^{O\rbr{\frac{1}{(1-2\eta)^4}}} \frac{1}{\epsilon} }$ for any $\eta \in [0,\frac12)$. Recent works of~\cite{zhang2020efficient} and~\cite{diakonikolas2020learning} provide passive learning algorithms with fully-polynomial sample complexities in this setting, achieving sample complexities of $O\rbr{ \frac{d}{\epsilon(1-2\eta)^5} }$ and $O\rbr{ \frac{d^9}{\epsilon^4 (1-2\eta)^{10}} }$ respectively.

In the distribution-free PAC learning setting, that is, when no assumptions are imposed on the unlabeled data distribution, efficient halfspace learning is much more challenging.
Recent breakthrough of~\cite{diakonikolas2019distribution} provides an efficient improper learner that can guarantee to output a halfspace with error $\eta+\epsilon$ with sample complexity $\poly(d,\frac1\epsilon)$. ~\cite{chen2020classification} improves over this result by proposing a proper learner, along with a generic ``distillation'' procedure that converts any improper learner to a proper one. In the same paper, they also show that for any statistical query algorithm, obtaining a classifier that achieves an error rate of $\opt + o(1)$ requires a superpolynomial number of statistical queries, where $\opt$ denotes the error rate of the Bayes-optimal halfspace $w^\star$. This lower bound is recently strengthened by~\cite{diakonikolas2020hardness}, showing that even achieving a weaker $\poly(\opt)$ error rate requires a superpolynomial number of statistical queries.

\paragraph{Efficient passive learning halfspaces with Tsybakov noise.} Recently,~\cite{diakonikolas2020learningb} obtains an algorithm with a quasi-polynomial time and sample complexity of $O\rbr{d^{O(\frac{1}{\alpha^2} \ln\frac{1}{\epsilon})}}$ for PAC learning halfspaces under $(A,\alpha)$-Tsybakov noise condition, under distributions with certain structural properties.
This result is further improved by~\cite{diakonikolas2020polynomial}, who obtain two algorithms with time and sample complexities of $O\rbr{\poly(d) \cdot (\frac1\epsilon)^{O(\frac{1}{\alpha^2})}}$ and $O\rbr{(\frac d \epsilon)^{O(\frac 1 \alpha)}}$ respectively. 

\paragraph{Efficient active learning halfspaces with Massart and Tsybakov noise.} By combining the agnostic halfspace learning algorithm of~\cite{kalai2008agnostically}, and margin-based sampling~\citep{balcan2007margin, balcan2013active},  \cite{awasthi2016learning} obtains an active halfspace learning algorithm that tolerates $\eta$-Massart noise with a sample complexity of $O\rbr{ d^{O(\frac{1}{(1-2\eta)^4})} \ln\frac1\epsilon }$ under the isotropic log-concavity assumption on the unlabeled data distribution. This result is recently substantially improved by~\cite{zhang2020efficient}, who obtain a label complexity of $O\rbr{ \frac{d}{(1-2\eta)^4} \polylog(\frac1\epsilon) }$ in the same setting; however this label complexity bound still does not match the information-theoretic lower bound of $\Omega\rbr{ \frac{d}{(1-2\eta)^2} \ln\frac1\epsilon }$. The only label-optimal result on PAC active halfspace learning under Massart noise we are aware of  is~\cite{yan2017revisiting}, however it relies on the strong assumption that the unlabeled data distribution is uniform over the unit sphere. Under Tsybakov noise condition, to the best of our knowledge, 
all prior active learning works require well-specified model assumptions on the conditional distribution of label given feature \citep{cesa2009robust,dekel2012selective,agarwal2013selective,krishnamurthy2017active}, e.g. assuming $\EE[ y \mid x ] = \sigma(w^\star \cdot x)$ for some known function $\sigma: \RR \to \RR$.
This is relatively strong, as it requires all examples $x$ that have the same projection on $w^\star$ to have the same value of $\eta(x)$. In contrast, our work does not require such assumptions. 


\section{Preliminaries}
\label{sec:setting}

We consider the standard PAC active learning for binary classfication setup~\citep{valiant1985learning,balcan2009agnostic}. Specifically, the instance space $\Xcal$ is $\RR^d$, the label space $\Ycal$ is $\cbr{-1,+1}$, and there is a data distribution $D$ supported on $\Xcal \times \Ycal$. The hypothesis class of interest is the set of linear classifiers, also known as halfspaces, defined as $\Hcal = \cbr{h_w: w \in \RR^d}$, where for every $w \in \RR^d$, $h_w$ denotes the corresponding linear classifier that maps $x$ to $\sign(\inner{w}{x})$. We use {\em error rate} to measure the performance of a classifier $h:  \Xcal \to \Ycal$, defined as $\err(h, D) = \PP_{(x,y) \sim D}(h(x) \neq y)$.
Given classifier $h$, and a set of labeled examples $S$, denote $\err(h, S) = \frac1{|S|} \sum_{(x,y) \in S} \ind\rbr{h(x) \neq y}$ as the empirical error rate of $h$ on $S$.
Throughout this paper, we assume that the Bayes-optimal classifier is a halfspace $h_{w^\star}$, where $w^\star \in \RR^d$ is a unit vector. It can be verified that $h_{w^\star}$ is indeed Bayes-optimal under Massart or Tsybakov noise conditions (recall Definitions \ref{def:mnc} and \ref{def:tnc}).

An active learning algorithm has access to two labeling oracles: first, an unlabeled example oracle $\EX$, which, upon query, returns an unlabeled example $x$ drawn from $D_X$, the marginal distribution of $D$ over $\Xcal$; second, a labeling oracle $\Ocal$, which, upon query with input example $x$, returns a label $y$ drawn from $D_{Y \mid X=x}$, the conditional distribution of $Y$ given $X=x$. A learner is said to achieve $(\epsilon,\delta)$-PAC active learning guarantee, if, by interactively querying the unlabeled example oracle $\EX$ and the labeling oracle $\Ocal$, it outputs a classifier $\tilde{h}$, such that with probability $1-\delta$, $\err(\tilde{h}, D) - \err(h_{w^\star}, D) \leq \epsilon$. Its label complexity is the total number of queries to $\Ocal$ throughout the learning process. 

For a natural number $N$, denote by $[N] := \cbr{1,2,\ldots,N}$.
For a vector $w$ in $\RR^d$, denote by its $\ell_2$-normalization $\hat{w} := \frac{w}{\|w\|_2}$ if $w \neq \vecz$, and $\hat{w} = (1,0,\ldots,0)$ if $w = \vecz$. {\em Throughout this paper, we reserve the ``hat'' symbol and notations such as $\hat{w}$, $\hat{v}$ for $\ell_2$-normalization and $\ell_2$-normalized vectors.} Unless explicitly stated, we use $\| \cdot \|$ to denote the vector $\ell_2$ norm. 
For two vectors $u$ and $v$, denote by $\theta(u,v) = \arccos \rbr{\frac{\inner{u}{v}}{\|u\|\|v\|}} \in [0,\pi]$ the angle between them; also, denote by $\tilde{\theta}(u,v) = \min(\theta(u,v), \pi - \theta(u,v)) \in [0, \frac \pi 2]$.
In our algorithm and analysis below, we will be frequently using the following definition of the distribution $D$ conditioned on a band: given a unit vector $\hat{w}$ and a threshold $b > 0$, denote by $B_{\hat{w},b} = \cbr{x \in \RR^d: \abr{\inner{\hat{w}}{x}} \leq b}$; in addition, denote by $D_{\hat{w}, b}$ the conditional distribution of $D$ on the set $\cbr{(x,y) \in \RR^d \times \cbr{-1,+1}: x \in B_{\hat{w},b}}$; similarly, denote by $D_{X \mid \hat{w}, b}$ the conditional distribution of $D_X$ on the set $B_{\hat{w},b}$.

Computational hardness results~\citep{chen2020classification,diakonikolas2020hardness} strongly suggest that efficient learning halfspaces with noise may be computationally intractable if no assumptions on the unlabeled data distribution are made, even under benign noise conditions such as Massart noise. Therefore, throughout this paper, following~\cite{diakonikolas2020polynomial}, we assume the unlabeled distribution to lie in a family of structured, or well-behaved distributions, defined as follows:

\begin{definition}[Well-behaved distributions~\citep{diakonikolas2020polynomial}]
\label{def:well-behaved}
Fix  $L, R, U, \beta > 0$. A distribution $D_X$ over $\RR^d$ is said to be $(2,L,R, U, \beta)$ well-behaved, if for any $2$-dimensional linear subspace $V$ of $\RR^d$, we have:
given an $x$ randomly drawn from $D_X$,
$x_V$, the projected coordinates of $x$ onto $V$\footnote{Formally, pick $(v_1,v_2)$ as an orthonormal basis of $V$; define $x_V := (\inner{v_1}{x}, \inner{v_2}{x}) \in \RR^2$.}, has a probability density function $p_V$ on $\RR^2$, such that:
\begin{enumerate}
    \item $p_V(z) \geq L$, for all $z$ such that $\| z \|_2 \leq R$;
    \item $p_V(z) \leq U$, for all $z \in \RR^2$;  
\end{enumerate}
in addition, for any unit vector $w$ in $\RR^d$ and any $t > 0$, $\PP_{D_X}(\abr{\inner{w}{x}} \geq t) \leq \exp(1-\frac{t}{\beta})$.
\end{definition}

The well-behavedness assumption captures the well-studied family of isotropic log-concave distributions~\citep{lovasz2007geometry,balcan2013active}, and can potentially be more general.\footnote{However, the well-behavedness assumption here does not capture the family of $s$-concave distributions ($s \geq -\frac{1}{2d+3}$) studied in recent works~\citep{balcan2017sample}, as it requires any 1-d projection of distribution to have sub-exponential tail. Whether our analysis can be extended to $s$-concave distributions is an interesting open question.}


In addition to Massart and Tsybakov noise conditions, we also study a subfamily of Tsybakov noise, namely geometric Tsybakov noise. Such noise assumption was first considered in nonparametric active learning literature~\citep{castro2008minimax}. It generalizes the ``strong Massart noise'' condition considered in~\cite{diakonikolas2020learning,zhang2017hitting}, in that it allows $\frac12 - \eta(x)$ to grow  polynomially with $\abr{\inner{w^\star}{x}}$, the distance between $x$ and the Bayes-optimal decision boundary $\cbr{x \in \RR^d: \inner{w^\star}{x} = 0}$.

\begin{definition}[Geometric Tsybakov noise condition]
\label{def:gtnc}
Given $B > 0$ and $\alpha \in (0,1]$, a distribution $D$ over $\RR^d \times \cbr{-1,+1}$ is said to satisfy the {\em $(B,\alpha)$-geometric Tsybakov noise condition} with respect to halfspace $w^\star$, if for all $x$ in $\RR^d$, $\frac12 - \eta(x) \geq \min\rbr{\frac12, B \abr{\inner{w^\star}{x}}^{\frac{1-\alpha}{\alpha}} }$, where $\eta(x) = \PP(y \neq \sign(\inner{w^\star}{x}) \mid x)$.
\end{definition}

It can be shown that, if the unlabeled distribution is well-behaved (Definition~\ref{def:well-behaved}),  modulo a logarithmic factor, $(B,\alpha)$-geometric Tsybakov noise condition implies $(A,\alpha)$-Tsybakov noise condition with $A = \tilde{O}\rbr{(\frac 1 B)^{\frac \alpha {1-\alpha}}}$; see Lemma~\ref{lem:geomtnc-tnc} in Appendix~\ref{sec:auxiliary} for a formal statement. 

\section{Algorithm}

We now describe our noise-tolerant active halfspace learning algorithm in detail. The main algorithm, Algorithm 1, has a simple structure: it first calls subprocedure $\initialize$ (line~\ref{step:init}) to generate a vector $v_1$ with $\ell_2$ distance at most $\frac{1}{4}$ to $w^\star$ with high probability. After obtaining $v_1$, it repeatedly calls $\optimize$ to refine its iterates $v_j$; as we will see, $\optimize$ guarantees that, with high probability, after iteration $j$, iterate $v_{j+1}$ is such that $\| v_{j+1} - w^\star \| \leq 4^{-(j+1)}$, i.e. after each iteration, an upper bound on $\|v_j - w^\star \|$ shrinks by a constant factor. 
The algorithm returns after its final iteration $k_\epsilon$ is finished; after this iteration, we have $\theta(v_{k_\epsilon+1}, w^\star) \leq r_{\epsilon} = \tilde{O}\rbr{ \epsilon }$, which implies that 
$\err(h_{\tilde{v}}, D) - \err(h_{w^\star}, D) \leq \PP_{D}(h_{\tilde{v}}(x) \neq h_{w^\star}(x)) \leq  \epsilon$ (see Lemma~\ref{lem:prob-angle} in Appendix \ref{sec:auxiliary}).

Depending on different noise conditions on $D$, we use different schedules of sampling region bandwidths $\cbr{b_j}$ and numbers of label queries $\cbr{T_j}$: 


{
\sloppy 
\begin{enumerate}
    \item Under the $\eta$-Massart noise condition, $b_j = b_{\MNC}(\eta, 4^{-(j+1)})$ and $T_j = T_{\MNC}(\eta, 4^{-(j+1)})$ for all $j \in \NN$, where $b_{\MNC}(\eta, r) = \tilde{\Theta}\rbr{ (1-2\eta) r}$ and $T_{\MNC}(\eta, r) = \tilde{O}\rbr{ \frac{d}{(1-2\eta)^2} (\ln\frac 1 {\delta r})^3 }$.
    
    \item Under the $(A,\alpha)$-Tsybakov noise condition for $\alpha \in (\frac12, 1]$, $b_j = b_{\TNC}(A, \alpha, 4^{-(j+1)})$ and $T_j = T_{\TNC}(A, \alpha, 4^{-(j+1)})$ for all $j \in \NN$, where  $b_{\TNC}(A, \alpha, r) = \tilde{\Theta}\rbr{ \min(r, (\frac{1}{A})^{\frac{1-\alpha}{2\alpha-1}} r^{\frac \alpha {2\alpha-1}}) }$ and $T_{\TNC}(A, \alpha, r) = \tilde{O}\rbr{ d  (\ln\frac 1 {\delta r})^3 \cdot (1 + (\frac{A}{r})^{\frac{2-2\alpha}{2\alpha-1}})}$.
    
    \item Under the $(B, \alpha)$-geometric Tsybakov noise condition, $b_j = b_{\GTNC}(B,\alpha, 4^{-(j+1)})$ and $T_j = T_{\GTNC}(B, \alpha, 4^{-(j+1)})$ for all $j \in \NN$, where $b_{\GTNC}(B, \alpha, r) = \tilde{\Theta}\rbr{ \min(r, B  r^{\frac1\alpha}) }$, and 
    $T_{\GTNC}(B, \alpha, r) = \tilde{O}\left( d (\ln\frac 1 {\delta r})^3 (1  + \frac{1}{B^2 } (\frac{1}{r})^{\frac{2-2\alpha}{\alpha}} )  \right)$.
\end{enumerate}
}

For brevity, in the above definitions, the dependence on the unlabeled distribution parameters $(L, R, U, \beta)$ is ignored; we refer the readers to Appendix~\ref{sec:params} for more precise definitions of these functions. 
The choices of $\cbr{b_j}$ and $\cbr{T_j}$ are to ensure that the algorithm's iterates $v_j$ are brought progressively closer to $w^\star$ with increasing $j$; this will be discussed in greater detail in Section~\ref{sec:theory}.


\begin{algorithm}[t]
\caption{Main algorithm}
\label{alg:main}
\begin{algorithmic}[1]
\REQUIRE{Bandwidth schedule $\cbr{b_j}$, iteration schedule $\cbr{T_j}$, target excess error $\epsilon \in (0,1)$, failure probability $\delta \in (0,1)$.}
\ENSURE{A halfspace $\tilde{v}$ such that $\err(h_{\tilde{v}}, D) - \err(h_{w^\star}, D) \leq \epsilon$.}
\STATE Define $r_\epsilon := \frac{\epsilon}{32 U \beta^2  (\ln\frac{12}{\epsilon})^2}$, and set $k_\epsilon := \lceil \log_4 \frac{1}{r_\epsilon} \rceil$ be the total number of iterations.
\STATE $v_1 \gets \initialize(\cbr{b_j}, \cbr{T_j})$.
\label{step:init}
\FOR{$j=1,\ldots,k_\epsilon$}
\STATE $v_{j+1} \gets \optimize(v_{j}, 4^{-(j+1)}, b_j, T_j, \avg)$.
\label{step:refine}
\ENDFOR
\RETURN $v_{k_\epsilon+1}$.
\end{algorithmic}
\end{algorithm}

We now discuss the two subprocedures employed by the main algorithm, \optimize and \initialize, in detail. 

\subsection{Procedure \optimize and its guarantees}
\label{subsec:optimize}

Procedure \optimize (Algorithm~\ref{alg:optimize}) aims at  refining its input halfspace $w_1$, so that it outputs a halfspace $\tilde{w}$ whose $\ell_2$ distance to $w^\star$ has an upper bound ($r$) at most a factor of $\frac14$ times the original $\ell_2$ distance upper bound between $w_1$ and  $w^\star$ ($4r$), with good probability. 

To this end, it maintains an iterate $w_t$; at each iteration, it performs adaptive sampling to obtain a labeled example $(x_t, y_t)$ drawn from $D_{\hat{w}_t, b}$, and updates with this new example using the well-known online gradient descent algorithm~\citep[e.g.][Chapter 11]{cesa2006prediction}. 
Following standard active learning sampling strategies~\cite[e.g.][]{balcan2009agnostic}, every draw from distributions $D_{X \mid \hat{w}_t, b}$ is done by rejection sampling, i.e., keep querying $\EX$ until it returns an example in $B_{\hat{w}_t,b}$.
After generating the iterates $\cbr{w_t}_{t=1}^T$, it aggregates them using either a normalize-and-average step (lines~\ref{line:avg-start} to~\ref{line:avg-end}), or draws a vector uniformly at random from this set, multiplied by a random sign (lines~\ref{line:rnd-start} to~\ref{line:rnd-end}), depending on its aggregation mode $\agg$. 


\optimize is similar to \refine  in~\cite{zhang2020efficient}, but has two key differences. First, the update vector $g_t$ used in our \optimize algorithm is $-y_t x_t$, whereas the update vector $g_t$ in \refine depends on the Massart noise parameter $\eta$; this undesirable dependence on $\eta$ implies that \refine cannot be used for handling broader noise conditions such as Tsybakov noise. Second, it allows two aggregation modes to be used, and when to use which aggregation mode depends on the precision of the input (i.e. the $\ell_2$ closeness of input $w_1$ and $w^\star$): as we will see, in early learning stages, we will call \optimize with  mode $\rnd$; in later stages, we will call \optimize with  mode $\avg$. \optimize is algorithmically similar to the nonconvex SGD algorithm of~\cite{diakonikolas2020learning} (see also earlier algorithmic insights of~\citet{guillory2009active}), who carefully construct a nonconvex learning objective such that under Massart noise, any stationary point of the objective corresponds to a vector close to $w^\star$; however, as we will see next, our analysis techniques are fairly different from theirs.



As its update rule suggest, \optimize performs online linear optimization with adaptively-chosen linear functions. Generalizing insights from prior work~\citep{zhang2020efficient}, our key observation is that, somewhat intriguingly, \optimize can be alternatively viewed as minimizing the following ``proximity function'' to $w^\star$:
\begin{definition}
Given distribution $D$ with its Bayes optimal classifier being a halfspace $h_{w^\star}$, and a positive number $b > 0$, define $\pot(w) := \EE_{D_{\hat{w},b}} \sbr{ (1-2\eta(x)) \abr{\inner{w^\star}{x}} }$.
\end{definition}

\begin{algorithm}[t]
\caption{$\optimize$}
\label{alg:optimize}
\begin{algorithmic}[1]
\REQUIRE vector $w_1$, target proximity $r$ (such that $\| w_1 - w^\star \| \leq 4 r$), bandwidth $b$, number of iterations $T$, aggregation mode $\agg$.
\ENSURE Optimized halfspace $\tilde{w}$ such that $\| \tilde{w} - w^\star \| \leq r$ with good probability.

\STATE Define $\Kcal = \cbr{w: \| w - w_1 \| \leq  4 r}$, step size $\alpha = \frac{r}{\beta} \sqrt{\frac{1}{dT}} / \rbr{\ln\frac{Td}{\delta r b R L}}$.
\FOR{$t=1, 2, \dots, T$}
\label{line:omd-loop-start}
\STATE Sample $x_t$ from $D_{X \mid \hat{w}_t, b}$ using rejection sampling, and query $\Ocal$ for its label $y_t$.
\STATE Update $w_{t+1} \gets\arg\min_{w \in \Kcal} \| w_t - \alpha g_t \|$,
where $g_t = - y_t x_t$. 
\label{line:omd}
\ENDFOR
\label{line:omd-loop-end}
\IF{$\agg = \avg$}
\label{line:avg-start}
\RETURN $\tilde{w} \gets \frac{1}{T} \sum_{t=1}^T \hat{w}_t$.
\label{line:avg-end}
\ELSIF{$\agg = \rnd$}
\label{line:rnd-start}
\RETURN $\tilde{w} \gets \sigma \cdot \hat{w}_\tau$, where $\tau$ is chosen uniformly at random from $[T]$, and $\sigma$ is chosen uniformly at random from $\cbr{-1,+1}$.
\label{line:rnd-end}
\ENDIF
\end{algorithmic}
\end{algorithm}

Recall that the three noise conditions considered in our paper all assume that $h_{w^\star}$ is the Bayes optimal classifier. Therefore, $\eta(x) \leq \frac12$ for all $x$, and consequently $\pot$ always takes nonnegative values.
In addition, $\pot$ is scale-invariant: as $B_{\widehat{\alpha w}, b} = B_{\hat{w}, b}$, $\pot(\alpha w) = \pot(w)$ for any $w \in \RR$ and $\alpha \neq 0$.
Informally, $\pot$ is a distance proxy function that measures the closeness of input $w$ and the optimal $w^\star$, although the closeness here is defined by a nonstandard measure $\tilde{\theta}(\cdot, \cdot)$ (recall its definition in Section~\ref{sec:setting}). This is formalized in the following lemma (proof in Appendix~\ref{sec:pot}):
\begin{lemma}
Suppose $D_X$ is $(2, L, R, U, \beta)$-well-behaved. In addition, $w \in \RR^d$ and $b > 0$ is such that $\tilde{\theta} = \tilde{\theta}(w, w^\star) \geq \Omega(b)$. Then:
\begin{enumerate}
    \item if $D$ satisfies $\eta$-Massart noise condition, $\pot(w) \geq \tilde{\Omega}\rbr{ (1-2\eta) \tilde{\theta} }$.
    \item If $D$ satisfies $(A, \alpha)$-Tsybakov noise condition, $\pot(w) \geq \Omega\rbr{ (\frac b A)^{\frac{1-\alpha}{\alpha}} \tilde{\theta} }$.

    \item If $D$ satisfies $(B, \alpha)$-geometric Tsybakov noise condition, then $\pot(w) \geq \tilde{\Omega}\rbr{ \min( \tilde{\theta}, B\tilde{\theta}^{\frac{1}{\alpha}}) }$.
    \end{enumerate}
    \label{lem:pot-lb-main-text}
\end{lemma}


The following key lemma formalizes the aforementioned claim that \optimize produces iterates $w_t$'s that approximately minimize $\pot$; specifically, the average value of $\pot(w_t)$'s is well-controlled, if the sampling bandwidth $b$ is small and the number of iterations $T$ is large.

\begin{lemma}
Suppose $D_X$ is $(2, L, R, U, \beta)$-well-behaved. There exists a numerical constant $c > 0$ such that the following holds. \optimize, with input initial vector $w_1$, target proximity $r \in (0,\frac 14]$ such that $\| w_1 - w^\star \| \leq 4r$, bandwidth $b \leq \frac R 2$, number of iterations $T$, produces iterates $\cbr{w_t}_{t=1}^T$, such that
with probability $1-\delta r$, 
\[
\frac1T \sum_{t=1}^T \pot(w_t)
\leq 
c \rbr{ b
+ 
(b + \beta r) \cdot \rbr{\ln\frac{Td}{\delta r b R L}} (\sqrt{\frac{d + \ln\frac1{\delta r}}{T}} + \frac{\ln\frac1{\delta r}}{T}) }.
\]
\label{lem:optimize-g}
\end{lemma}



The proof of Lemma~\ref{lem:optimize-g} can be found at Appendix~\ref{sec:optimize-g}.
Its key insight is similar to the ideas in~\cite{zhang2020efficient}: we derive a regret guarantee of the online linear optimization problem induced by the adaptively-chosen gradient vectors $\cbr{g_t}_{t=1}^T$, which implies an upper bound on the {\em negative benchmark} term. Thanks to the adaptive sampling scheme, the negative benchmark concentrates to $\sum_{t=1}^T \pot(w_t)$. Although the proof of Lemma~\ref{lem:optimize-g}  uses standard regret results on online linear optimization, it is not a direct consequence of the standard reduction from online convex optimization to online linear optimization -- $\pot(w)$ is not necessarily convex in $w$.

We now discuss the aggregation mode $\agg$ in more detail. As \optimize is called by the main algorithm (Algorithm~\ref{alg:main}) repeatedly, we discuss its different settings in earlier and later stages of the main algorithm respectively.

In later stages of calling $\optimize$ (specifically, when its input target proximity $r \leq \frac{1}{16}$), the constraint set $\Kcal$ ensures that all $w_t$'s have acute angles with $w^\star$; in this case, for all $t$, $\tilde{\theta}(w_t, w^\star) = \theta(w_t, w^\star)$, and therefore  upper bounds on $\pot(w_t)$ imply upper bounds on $\theta(w_t, w^\star)$.
In this ``local convergence'' regime, applying the guarantees provided by Lemma~\ref{lem:optimize-g}, a deterministic average over the normalized iterates $\hat{w}_t$'s achieves the target $\ell_2$ proximity to $w^\star$; see item~\ref{item:avg} of Lemma~\ref{lem:optimize-main} or Lemma \ref{lem:optimize-main-avg-only} for a precise statement. This corresponds to the $\avg$ mode (lines~\ref{line:avg-start} to~\ref{line:avg-end}). 

In contrast, in early stages of calling \optimize (specifically, when $r > \frac1{16}$), its iterates $\cbr{w_t}_{t=1}^T$ may not yet be in acute angle with $w^\star$; in this case, it is hard to guarantee that the average normalized iterate $\frac{1}{T} \sum_{t=1}^T \hat{w}_t$ is close to $w^\star$. This motivates the second mode $\rnd$ (lines~\ref{line:rnd-start} to~\ref{line:rnd-end}), which returns a vector uniformly at random from $\hat{w}_t$'s, times a random sign $\sigma$. By the guarantees of Lemma~\ref{lem:optimize-g}, it can be shown that with appropriate settings of parameters $\alpha, T, b$, \optimize guarantees to return $\tilde{w}$ with a constant $\ell_2$ distance to $w^\star$ with a constant probability; see item~\ref{item:rnd} of Lemma~\ref{lem:optimize-main} for a precise statement.


\subsection{Procedure \initialize and its guarantees}
Procedure \initialize (Algorithm~\ref{alg:initialize}) aims at label-efficiently learning a halfspace $\hat{u}_0$ amenable to local refinement: it guarantees that with high probability, the output halfspace $\hat{u}_0$ satisfies that $\| \hat{u}_0 - w^\star \| \leq \frac14$. 

At a high level, \initialize uses \optimize as a black box in a label efficient manner. 
Recall from the previous subsection that, with appropriate settings of input parameters, running \optimize with aggregation mode $\rnd$ guarantees to output a halfspace $\frac14$-close to $w^\star$ with  constant probability. \initialize   ``boosts'' the above guarantee, in that it increases the probability of outputting a vector $\frac14$-close to $w^\star$ from a small constant to $1-O(\delta)$. 



\initialize consists of two stages. In the first stage (lines~\ref{line:stage-1-start} to~\ref{line:stage-1-end}), it generates $U$, a set of halfspaces of size $N = \lceil 10 \ln\frac 4 \delta \rceil$, such that at least one  element in $U$ has a small excess error rate, specifically $O(\epsilon_0)$; here the choice of $\epsilon_0$ depends on the noise condition of $D$, as will be discussed next. To achieve this target excess error rate, it runs \optimize for $k_0 = \tilde{O}(\log\frac{1}{\epsilon_0})$ iterations; Claim~\ref{claim:stage1} in Appendix~\ref{sec:initialize} shows that each trial $i$ generates $v_{i,k_0+1}$ with excess error $O(\epsilon_0)$ with constant probability. Because of the independence of the $N$ trials, with high probability, one of the $v_{i,k_0+1}$'s will have excess error $O(\epsilon_0)$.


\begin{algorithm}[t]
\caption{\initialize: a label efficient acute initialization procedure}
\label{alg:initialize}
\begin{algorithmic}[1]
\REQUIRE{Bandwidth schedule $\cbr{b_j}$,  iterations schedule $\cbr{T_j}$.}
\ENSURE{Unit vector $\hat{u}_0$ such that $\| \hat{u}_0 - w^\star \| \leq \frac14$.}



\STATE Define $r_0 = \frac{\epsilon_0}{64 U \beta^2  (\ln\frac{24}{\epsilon_0})^2}$, where $\epsilon_0$ is defined according to the noise condition of $D$; $k_0 = \lceil \log_4 (\frac{1}{r_0}) \rceil$ is the number of iterations per trial; $N = \lceil 10 \ln \frac {4}{\delta} \rceil$ the number of trials in the first stage.

\FOR{$i=1,2,\ldots,N$}
\label{line:stage-1-start}

\STATE Initialize $v_{i,0} \gets \vecz$.


\FOR{$j=0,1,2,\ldots,k_0$}

\STATE Set $\agg \gets \rnd$ if $j = 0$; $\agg \gets \avg$ otherwise.


\STATE $v_{i,j+1} \gets \optimize(v_{i,j}, 4^{-(j+1)}, b_j, T_j, \agg)$.
\ENDFOR
\ENDFOR

\STATE $U \gets \cbr{v_{i,k_0+1}: i \in [N]}$. 
\label{line:stage-1-end}

\STATE $S \leftarrow$ draw $O(\frac{1}{\epsilon_0^2}\ln \frac {N}{\delta} )$ random unlabeled examples from $D_X$, and query $\Ocal$ for their labels.
\label{line:stage-2-start}

\RETURN $\hat{u}_0 = \frac{u_0}{\| u_0\|}$, where $u_0 = \argmin_{u \in U} \err(h_u, S)$.
\label{line:stage-2-end}
\end{algorithmic}
\end{algorithm}

In the second stage (lines~\ref{line:stage-2-start} to~\ref{line:stage-2-end}), it draws $S$, a set of labeled examples from $D$, and selects the halfspace $\hat{u}$ in $U$ with the smallest empirical error on $S$.
Combining the guarantees of $U$ in the first stage, the choice of $|S| = \tilde{O}(\epsilon_0^{-2})$, and standard guarantees of empirical risk minimization, it is guaranteed that $\hat{u}$ has excess error  $\tilde{O}(\epsilon_0)$ with high probability. 
Parameter $\epsilon_0$ is set to ensure that the above excess error guarantee can be translated to a geometric $\ell_2$-proximity guarantee, and therefore depends on different noise conditions on $D$. Specifically, we set $\epsilon_0$ as:
\begin{enumerate}
    \item $\epsilon_{\MNC}(\eta) = \tilde{O}\rbr{ 1-2\eta }$, under the $\eta$-Massart noise condition;
    
    \item $\epsilon_{\TNC}(A, \alpha) = \tilde{O}\rbr{  (\frac{1}{A})^{\frac{1-\alpha} {\alpha}} }$, under the $(A, \alpha)$-Tsybakov noise condition with $\alpha \in (\frac12, 1]$;
    
    \item $\epsilon_{\GTNC}(B, \alpha) = \tilde{O}\rbr{ B }$, under the $(B,\alpha)$-geometric Tsybakov noise condition. 
\end{enumerate}
We again refer the readers to Appendix~\ref{sec:params} for more precise definitions of these functions, with their dependence on $(L, R, U, \beta)$ explicit. 



We show the following guarantee of \initialize, under any one of the three noise conditions considered.
\begin{lemma}
Suppose $D$ is $(2, L, R, U, \beta)$-well behaved and satisfies one of the three noise conditions. With the respective settings of $\cbr{b_j}, \cbr{T_j}, \epsilon_0$, \initialize  outputs a unit vector $\hat{u}_0$, such that with probability $1-\delta/2$, $\| \hat{u}_0 - w^\star \| \leq \frac{1}{4}$. The total number of label queries by \initialize is at most:
\begin{enumerate}
    \item $\tilde{O} \rbr{ \frac{d}{(1-2\eta)^2} }$, if $D$ satisfies $\eta$-Massart noise;
    \item $\tilde{O} \rbr{ d \cdot \rbr{ 1 + A^{\frac{2-2\alpha}{\alpha(2\alpha-1)}} } }$, if $D$ satisfies $(A, \alpha)$-Tsybakov noise with $\alpha \in (\frac12, 1]$;
    \item $\tilde{O} \rbr{ d \cdot \rbr{ 1 + (\frac{1}{B})^{\frac 2 \alpha} } }$, if $D$ satisfies $(B,\alpha)$-geometric Tsybakov noise.
\end{enumerate} 
\label{lem:initialize}
\end{lemma}


Specifically, under the $\eta$-Massart noise condition, this yields a procedure that can  output a vector $\hat{u}$ with constant proximity to the optimal halfspace $w^\star$ with high probability, using $\tilde{O}(\frac{d}{(1-2\eta)^2})$ label queries. 
This label complexity matches the $\Omega(\frac{d}{(1-2\eta)^2})$ information-theoretic lower bound~\cite[e.g.][Theorem 1]{yan2017revisiting}. When specialized to isotropic log-concave unlabeled distribution settings, this resolves an open problem by~\cite{zhang2020efficient} on whether there is an efficient and label-optimal initialization procedure that  reliably computes a vector with small constant angle with $w^\star$.


\section{Performance guarantees}
\label{sec:theory}

We now present Theorem~\ref{thm:main}, the main result of this paper.

\begin{theorem}
Fix $\epsilon \in (0,1)$ and $\delta \in (0, \frac{1}{10})$. Suppose $D$ is $(2, L, R, U, \beta)$-well behaved and satisfies one of the three noise conditions. With the settings of $\cbr{b_j}, \cbr{T_j}$, and $\epsilon_0$ under the respective noise conditions, with probability $1-\delta$, Algorithm~\ref{alg:main} outputs a halfspace $\tilde{v}$, such that $\err(h_{\tilde{v}}, D) - \err(h_{w^\star}, D) \leq \epsilon$. In addition, its total number of label queries is at most: 
\begin{enumerate}
    \item $\tilde{O}\rbr{ \frac{d}{(1-2\eta)^2} \polylog(\frac1\epsilon)}$, if $D$ satisfies $\eta$-Massart noise;
    \item $\tilde{O}\rbr{ d \cdot \rbr{1 + A^{\frac{2-2\alpha}{\alpha(2\alpha-1)}} + (\frac{A}{\epsilon})^{\frac{2-2\alpha}{2\alpha-1}}} }$, if $D$ satisfies $(A, \alpha)$-Tsybakov noise with $\alpha \in (\frac12, 1]$;
    \item $\tilde{O}\rbr{ d \cdot \rbr{ 1 + (\frac{1}{B})^{\frac 2 \alpha} + \frac{1}{B^2} (\frac 1 \epsilon)^{\frac{2-2\alpha}{\alpha}}} }$, if $D$ satisfies $(B,\alpha)$-geometric Tsybakov noise.
\end{enumerate}
\label{thm:main}
\end{theorem}
Specialized to the $\eta$-Massart noise condition, our label complexity bound matches information-theoretic lower bounds~\cite[e.g.][Theorem 1]{yan2017revisiting} up to polylogarithmic factors, and improves over the state-of-the-art active halfspace learning algorithm of~\cite{zhang2020efficient} in two aspects. First, our label complexity is a factor of $O\rbr{ \frac{1}{(1-2\eta)^2} }$ lower than~\cite{zhang2020efficient}, thanks to the new initialization procedure; second, our algorithm and analysis allows for dealing with a broader set of unlabeled data distributions beyond isotropic log-concave, matching the assumptions employed in recent works~\citep[e.g.][]{diakonikolas2020polynomial}. 

Under the $(A, \alpha)$-Tsybakov noise condition, our theorem provides nontrivial label complexity results when $\alpha \in (\frac12, 1]$. In the extreme case when $\alpha = 1$, our algorithm has a label complexity of $\tilde{O}\rbr{ d \polylog(\frac1\epsilon) }$. The label complexity bound becomes higher when $\alpha$ is further away from $1$. Compared to the recent passive learning algorithm of~\cite{diakonikolas2020polynomial} that can tolerate $(A,\alpha)$-Tsybakov noise for any $\alpha \in (0,1]$, our results cannot allow $\alpha$ to be in $(0,\frac12]$, but our algorithm has better label efficiency when $\alpha$ is close to 1.

Under the $(B, \alpha)$-geometric Tsybakov noise condition, our label complexity bound $\tilde{O}\rbr{ d (\frac{1}{\epsilon})^{\frac{2-2\alpha}{\alpha}} }$ is higher than the computationally inefficient algorithm of~\cite{balcan2013active}, which has a label complexity of $\tilde{O}\rbr{ d (\frac1\epsilon)^{2-2\alpha} }$.
This is due to a limitation of our current proof technique: we reduce the goal of achieving excess error guarantee to achieving geometric proximity. 
Our proof in fact yields a stronger result: with $\tilde{O}\rbr{ d (\frac{1}{\epsilon})^{\frac{2-2\alpha}{\alpha}} }$ label queries, our algorithm outputs a halfspace that has angle $O(\epsilon)$ with $w^\star$ with high probability; this result matches the information-theoretic lower bound of~\cite{wang2016noise} in achieving closeness-in-angle guarantees, in the dependence on $\epsilon$. We leave whether it is possible to develop  efficient active learning algorithms with label complexity guarantees matching those of computationally inefficient algorithms in this setting as an important open question.

\begin{remark}[Unlabeled sample complexity of Algorithm~\ref{alg:main}]
Our active learning algorithm consumes a total of $O(\poly(d, \frac1\epsilon))$ unlabeled examples with high probability.
To see this, note that our sampling regions' bandwidths all satisfy $b = \Omega(\poly(\epsilon))$ under each of the three noise conditions and therefore have probability masses $\Omega(\poly(\epsilon))$ by Lemma~\ref{lem:1d-prob-ub}. This implies that with high probability, each active sampling invokes at most $O(\poly(\frac1\epsilon))$ calls to the unlabeled example oracle $\EX$; this implies that the total number of calls to $\EX$ is also $O(\poly(d, \frac1\epsilon))$.
\end{remark}

\begin{remark}[Attribute efficiency]
Our algorithm and analysis can be straightforwardly modified to achieve attribute efficiency \citep[e.g.][]{littlestone1987learning,blum1990learning,awasthi2016learning,zhang2018efficient}, i.e. achieving label complexities that exploit the sparsity of the Bayes-optimal linear classifier $w^\star$. Specifically, under the extra assumption that $w^\star$ is $s$-sparse ($s \ll d$), a variant of our algorithm achieves a guarantee similar to  Theorem~\ref{thm:main}, with the dimension $d$ in the label complexity bounds replaced with $s \polylog(d)$. We provide the details in Appendix~\ref{sec:sparsity}.
\end{remark}

\subsection{Proof sketch of Theorem~\ref{thm:main}}
We now outline the proof of Theorem~\ref{thm:main}. 
Recall that from Lemma~\ref{lem:initialize}, line~\ref{step:init} of the main algorithm calls \initialize to generate vector $v_1$ such that $\| v_1 - w^\star \| \leq \frac14$ with probability $1-\delta/2$. This step uses $\tilde{O}\rbr{\frac{d}{(1-2\eta)^2}}$, $\tilde{O}\rbr{d \cdot (1 + A^{\frac{2-2\alpha}{\alpha(2\alpha-1)}}) }$, and $\tilde{O}\rbr{ d \cdot (1 + (\frac{1}{B})^{\frac 2 \alpha})}$ label queries to $\Ocal$, under the three noise conditions respectively. 

For the guarantees in subsequent rounds, we rely on the following lemma, which shows that repeatedly applying \optimize yields local convergence guarantees.
Specifically, this lemma implies that, given an input halfspace $v_j \in \RR^d$ such that $\| v_j - w^\star \| \leq 4^{-j}$ at the beginning of the $j$-th iteration of the main algorithm (Algorithm~\ref{alg:main}), \optimize, with settings of bandwidth parameter $b_j$ and number of iterations $T_j$, outputs a refined halfspace $v_{j+1}$ such that its $\ell_2$ distance with $w^\star$ is at most $4^{-(j+1)}$ with high probability. 


\begin{lemma}
Fix $r \in (0,\frac{1}{16}]$, and $\delta \in (0,\frac1{10})$.
Suppose $D$ is $(2, L, R, U, \beta)$-well behaved and satisfies one of the three noise conditions; in addition, 
$b$ and $T$ are such that:
\begin{enumerate}
    \item $b = b_{\MNC}(\eta, r), T = T_{\MNC}(\eta, r)$, if $D$ satisfies $\eta$-Massart noise;
    \item $b = b_{\TNC}(A, \alpha, r), T = T_{\TNC}(A, \alpha, r)$, if $D$ satisfies $(A, \alpha)$-Tsybakov noise with $\alpha \in (\frac12, 1]$;
    \item $b = b_{\GTNC}(B, \alpha, r), T = T_{\GTNC}(B, \alpha, r)$, if $D$ satisfies $(B, \alpha)$-geometric Tsybakov noise.
\end{enumerate}
Then \optimize, with input initial $w_1$ satisfying $\| w_1 - w^\star \| \leq 4 r$, target proximity $r$, bandwidth $b$, number of iterations $T$, aggregation method $\agg = \avg$, outputs $\tilde{w}$ such that probability $1-\delta r$,
    $\| \tilde{w} - w^\star \| \leq r$.
    \label{lem:optimize-main-avg-only}
\end{lemma}

The proof of Lemma~\ref{lem:optimize-main-avg-only} can be found at Appendix~\ref{sec:optimize}. 
Some intuitions on this lemma have been given in Section~\ref{subsec:optimize}, and we elaborate on its key ideas in greater detail here. Recall that 
Lemma~\ref{lem:optimize-g} shows that running \optimize gives an upper bound on $\frac 1 T \sum_{t=1}^T \pot(w_t)$, the average value of $\pot(w_t)$'s, in terms of $b$ and $T$. 
We set $b$ and $T$ differently under different noise conditions, so that $\frac 1 T \sum_{t=1}^T \pot(w_t)$ can be controlled at an appropriate level. 
Markov's Inequality implies that there is an overwhelming fraction ($\geq \frac{31}{32}$) of $w_t$'s with small $\pot(w_t)$ -  denote by $S$ the set of such $t$'s. Now, we conduct a case analysis:
\begin{enumerate}
\item For every $t$ in $S$, $\pot(w_t)$ is small. Recall that  Lemma~\ref{lem:pot-lb-main-text} shows that a small value of $\pot(w_t)$ implies a small value of $\tilde{\theta}(w_t, w^\star)$. 
In addition, the diameter of the constraint set $\Kcal$ is at most $8 r \leq \frac12$, and both $w_t$ and $w^\star$ are in $\Kcal$, so $\theta(w_t, w^\star)$ is acute (see Lemma~\ref{lem:angle-l2}) and is equal to $\tilde{\theta}(w_t, w^\star)$.
Hence, for all $t$ in $S$, $\theta(w_t, w^\star) \leq \frac{r}{2}$ and consequently, $\| \hat{w}_t - w^\star \| \leq \frac r 2$. 
\item On the other hand, for every $t$ in $[T] \setminus S$, we still have $w_t \in \Kcal$, so $\| \hat{w}_t - w^\star \|$ is  at most $16 r$. 
\end{enumerate}

By averaging over the upper bounds on $\| \hat{w}_t - w^\star \|$, and using the convexity of $\ell_2$ norm, we conclude that $\|\tilde{w} - w^\star \| = \| \frac1T \sum_{t=1}^T \hat{w}_t - w^\star \| \leq \frac1T \sum_{t=1}^T \| \hat{w}_t - w^\star \| \leq r$.


Equipped with the above initialization and local convergence guarantees, the proof of Theorem~\ref{thm:main} is now straightforward; its details can be found at Appendix~\ref{sec:proof-main}.


\section{Conclusions and open problems}
We provide an efficient active halfspace learning algorithm that can achieve new label complexity guarantees under Massart and Tsybakov noise conditions, under certain structural assumptions on the unlabeled data distribution. Specifically, our algorithm achieves a near-optimal label complexity under the Massart noise condition, and achieves new label complexity guarantees under two subfamilies of Tsybakov noise conditions. A key open problem is to develop efficient algorithms with label complexities matching those of computationally inefficient approaches under  $(A,\alpha)$-Tsybakov noise, for all $\alpha \in (0,1]$.
Another interesting open question is to design efficient active learning algorithms that can adapt to unknown noise conditions.


\paragraph{Acknowledgments.} We thank Yining Wang for helpful discussions on label complexity lower bounds in~\citep{wang2016noise} for active learning halfspaces under Tsybakov noise under the uniform distribution. We also thank the anonymous reviewers for their constructive feedback. 


    


\bibliography{learning}

\begin{thebibliography}{50}
\providecommand{\natexlab}[1]{#1}
\providecommand{\url}[1]{\texttt{#1}}
\expandafter\ifx\csname urlstyle\endcsname\relax
  \providecommand{\doi}[1]{doi: #1}\else
  \providecommand{\doi}{doi: \begingroup \urlstyle{rm}\Url}\fi

\bibitem[Agarwal(2013)]{agarwal2013selective}
Alekh Agarwal.
\newblock Selective sampling algorithms for cost-sensitive multiclass
  prediction.
\newblock In \emph{International Conference on Machine Learning}, pages
  1220--1228. PMLR, 2013.

\bibitem[Arora et~al.(1997)Arora, Babai, Stern, and Sweedyk]{arora1997hardness}
Sanjeev Arora, L{\'a}szl{\'o} Babai, Jacques Stern, and Z~Sweedyk.
\newblock The hardness of approximate optima in lattices, codes, and systems of
  linear equations.
\newblock \emph{Journal of Computer and System Sciences}, 54\penalty0
  (2):\penalty0 317--331, 1997.

\bibitem[Awasthi et~al.(2015)Awasthi, Balcan, Haghtalab, and
  Urner]{awasthi2015efficient}
Pranjal Awasthi, Maria{-}Florina Balcan, Nika Haghtalab, and Ruth Urner.
\newblock Efficient learning of linear separators under bounded noise.
\newblock In \emph{Proceedings of the 28th Annual Conference on Learning
  Theory}, pages 167--190, 2015.

\bibitem[Awasthi et~al.(2016)Awasthi, Balcan, Haghtalab, and
  Zhang]{awasthi2016learning}
Pranjal Awasthi, Maria{-}Florina Balcan, Nika Haghtalab, and Hongyang Zhang.
\newblock Learning and 1-bit compressed sensing under asymmetric noise.
\newblock In \emph{Proceedings of the 29th Conference on Learning Theory},
  pages 152--192, 2016.

\bibitem[Balcan and Feldman(2013)]{balcan2013statistical}
Maria{-}Florina Balcan and Vitaly Feldman.
\newblock Statistical active learning algorithms.
\newblock In \emph{Proceedings of the 27th Annual Conference on Neural
  Information Processing Systems}, pages 1295--1303, 2013.

\bibitem[Balcan and Haghtalab(2020)]{balcan2020noise}
Maria-Florina Balcan and Nika Haghtalab.
\newblock Noise in classification.
\newblock \emph{Beyond the Worst-Case Analysis of Algorithms}, page 361, 2020.

\bibitem[Balcan and Long(2013)]{balcan2013active}
Maria{-}Florina Balcan and Philip~M. Long.
\newblock Active and passive learning of linear separators under log-concave
  distributions.
\newblock In \emph{Proceedings of the 26th Annual Conference on Learning
  Theory}, pages 288--316, 2013.

\bibitem[Balcan and Zhang(2017)]{balcan2017sample}
Maria-Florina Balcan and Hongyang Zhang.
\newblock Sample and computationally efficient learning algorithms under
  s-concave distributions.
\newblock In \emph{Advances in Neural Information Processing Systems}, pages
  4796--4805, 2017.

\bibitem[Balcan et~al.(2007)Balcan, Broder, and Zhang]{balcan2007margin}
Maria{-}Florina Balcan, Andrei~Z. Broder, and Tong Zhang.
\newblock Margin based active learning.
\newblock In \emph{Proceedings of the 20th Annual Conference on Learning
  Theory}, pages 35--50, 2007.

\bibitem[Balcan et~al.(2009)Balcan, Beygelzimer, and
  Langford]{balcan2009agnostic}
Maria-Florina Balcan, Alina Beygelzimer, and John Langford.
\newblock Agnostic active learning.
\newblock \emph{Journal of Computer and System Sciences}, 75\penalty0
  (1):\penalty0 78--89, 2009.

\bibitem[Bartlett et~al.(2006)Bartlett, Jordan, and
  McAuliffe]{bartlett2006convexity}
Peter~L Bartlett, Michael~I Jordan, and Jon~D McAuliffe.
\newblock Convexity, classification, and risk bounds.
\newblock \emph{Journal of the American Statistical Association}, 101\penalty0
  (473):\penalty0 138--156, 2006.

\bibitem[Beygelzimer et~al.(2010)Beygelzimer, Hsu, Langford, and
  Zhang]{beygelzimer2010agnostic}
Alina Beygelzimer, Daniel~J Hsu, John Langford, and Tong Zhang.
\newblock Agnostic active learning without constraints.
\newblock \emph{Advances in Neural Information Processing Systems},
  23:\penalty0 199--207, 2010.

\bibitem[Blum(1990)]{blum1990learning}
Avrim Blum.
\newblock Learning boolean functions in an infinite attribute space.
\newblock In \emph{Proceedings of the 22nd Annual {ACM} Symposium on Theory of
  Computing}, pages 64--72, 1990.

\bibitem[Blum et~al.(1996)Blum, Frieze, Kannan, and
  Vempala]{blum1996polynomial}
Avrim Blum, Alan~M. Frieze, Ravi Kannan, and Santosh~S. Vempala.
\newblock A polynomial-time algorithm for learning noisy linear threshold
  functions.
\newblock In \emph{Proceedings of the 37th Annual Symposium on Foundations of
  Computer Science}, pages 330--338, 1996.

\bibitem[Castro and Nowak(2008)]{castro2008minimax}
Rui~M Castro and Robert~D Nowak.
\newblock Minimax bounds for active learning.
\newblock \emph{IEEE Transactions on Information Theory}, 54\penalty0
  (5):\penalty0 2339--2353, 2008.

\bibitem[Cesa-Bianchi and Lugosi(2006)]{cesa2006prediction}
Nicolo Cesa-Bianchi and Gabor Lugosi.
\newblock \emph{Prediction, learning, and games}.
\newblock Cambridge university press, 2006.

\bibitem[Cesa-Bianchi et~al.(2009)Cesa-Bianchi, Gentile, and
  Orabona]{cesa2009robust}
Nicolo Cesa-Bianchi, Claudio Gentile, and Francesco Orabona.
\newblock Robust bounds for classification via selective sampling.
\newblock In \emph{Proceedings of the 26th annual international conference on
  machine learning}, pages 121--128, 2009.

\bibitem[Chen et~al.(2020)Chen, Koehler, Moitra, and
  Yau]{chen2020classification}
Sitan Chen, Frederic Koehler, Ankur Moitra, and Morris Yau.
\newblock Classification under misspecification: Halfspaces, generalized linear
  models, and connections to evolvability.
\newblock \emph{arXiv preprint arXiv:2006.04787}, 2020.

\bibitem[Cortes and Vapnik(1995)]{cortes1995support}
Corinna Cortes and Vladimir Vapnik.
\newblock Support-vector networks.
\newblock \emph{Machine learning}, 20\penalty0 (3):\penalty0 273--297, 1995.

\bibitem[Cristianini and Shawe{-}Taylor(2010)]{cristianini2000introduction}
Nello Cristianini and John Shawe{-}Taylor.
\newblock \emph{An Introduction to Support Vector Machines and Other
  Kernel-based Learning Methods}.
\newblock Cambridge University Press, 2010.

\bibitem[Daniely(2016)]{daniely2016complexity}
Amit Daniely.
\newblock Complexity theoretic limitations on learning halfspaces.
\newblock In \emph{Proceedings of the forty-eighth annual ACM symposium on
  Theory of Computing}, pages 105--117, 2016.

\bibitem[Dekel et~al.(2012)Dekel, Gentile, and Sridharan]{dekel2012selective}
Ofer Dekel, Claudio Gentile, and Karthik Sridharan.
\newblock Selective sampling and active learning from single and multiple
  teachers.
\newblock \emph{The Journal of Machine Learning Research}, 13\penalty0
  (1):\penalty0 2655--2697, 2012.

\bibitem[Diakonikolas and Kane(2020)]{diakonikolas2020hardness}
Ilias Diakonikolas and Daniel~M Kane.
\newblock Hardness of learning halfspaces with massart noise.
\newblock \emph{arXiv preprint arXiv:2012.09720}, 2020.

\bibitem[Diakonikolas et~al.(2019)Diakonikolas, Gouleakis, and
  Tzamos]{diakonikolas2019distribution}
Ilias Diakonikolas, Themis Gouleakis, and Christos Tzamos.
\newblock Distribution-independent {PAC} learning of halfspaces with massart
  noise.
\newblock In \emph{Proceedings of the 33rd Annual Conference on Neural
  Information Processing Systems}, pages 4751--4762, 2019.

\bibitem[Diakonikolas et~al.(2020{\natexlab{a}})Diakonikolas, Kane, Kontonis,
  Tzamos, and Zarifis]{diakonikolas2020polynomial}
Ilias Diakonikolas, Daniel~M Kane, Vasilis Kontonis, Christos Tzamos, and Nikos
  Zarifis.
\newblock A polynomial time algorithm for learning halfspaces with tsybakov
  noise.
\newblock \emph{arXiv preprint arXiv:2010.01705}, 2020{\natexlab{a}}.

\bibitem[Diakonikolas et~al.(2020{\natexlab{b}})Diakonikolas, Kane, and
  Zarifis]{diakonikolas2020near}
Ilias Diakonikolas, Daniel~M Kane, and Nikos Zarifis.
\newblock Near-optimal sq lower bounds for agnostically learning halfspaces and
  relus under gaussian marginals.
\newblock \emph{arXiv preprint arXiv:2006.16200}, 2020{\natexlab{b}}.

\bibitem[Diakonikolas et~al.(2020{\natexlab{c}})Diakonikolas, Kontonis, Tzamos,
  and Zarifis]{diakonikolas2020learning}
Ilias Diakonikolas, Vasilis Kontonis, Christos Tzamos, and Nikos Zarifis.
\newblock Learning halfspaces with massart noise under structured
  distributions.
\newblock \emph{arXiv preprint arXiv:2002.05632}, 2020{\natexlab{c}}.

\bibitem[Diakonikolas et~al.(2020{\natexlab{d}})Diakonikolas, Kontonis, Tzamos,
  and Zarifis]{diakonikolas2020learningb}
Ilias Diakonikolas, Vasilis Kontonis, Christos Tzamos, and Nikos Zarifis.
\newblock Learning halfspaces with tsybakov noise.
\newblock \emph{arXiv preprint arXiv:2006.06467}, 2020{\natexlab{d}}.

\bibitem[Feldman et~al.(2006)Feldman, Gopalan, Khot, and
  Ponnuswami]{feldman2006new}
Vitaly Feldman, Parikshit Gopalan, Subhash Khot, and Ashok~Kumar Ponnuswami.
\newblock New results for learning noisy parities and halfspaces.
\newblock In \emph{Proceedings of the 47th Annual IEEE Symposium on Foundations
  of Computer Science}, pages 563--574, 2006.

\bibitem[Guillory et~al.(2009)Guillory, Chastain, and
  Bilmes]{guillory2009active}
Andrew Guillory, Erick Chastain, and Jeff Bilmes.
\newblock Active learning as non-convex optimization.
\newblock In \emph{Artificial Intelligence and Statistics}, pages 201--208.
  PMLR, 2009.

\bibitem[Guruswami and Raghavendra(2009)]{guruswami2009hardness}
Venkatesan Guruswami and Prasad Raghavendra.
\newblock Hardness of learning halfspaces with noise.
\newblock \emph{SIAM Journal on Computing}, 39\penalty0 (2):\penalty0 742--765,
  2009.

\bibitem[Hanneke(2011)]{hanneke2011rates}
Steve Hanneke.
\newblock Rates of convergence in active learning.
\newblock \emph{The Annals of Statistics}, 39\penalty0 (1):\penalty0 333--361,
  2011.

\bibitem[Hanneke(2014)]{hanneke2014theory}
Steve Hanneke.
\newblock Theory of disagreement-based active learning.
\newblock \emph{Foundations and Trends{\textregistered} in Machine Learning},
  7\penalty0 (2-3):\penalty0 131--309, 2014.

\bibitem[Kalai et~al.(2008)Kalai, Klivans, Mansour, and
  Servedio]{kalai2008agnostically}
Adam~Tauman Kalai, Adam~R Klivans, Yishay Mansour, and Rocco~A Servedio.
\newblock Agnostically learning halfspaces.
\newblock \emph{SIAM Journal on Computing}, 37\penalty0 (6):\penalty0
  1777--1805, 2008.

\bibitem[Klivans and Kothari(2014)]{klivans2014embedding}
Adam Klivans and Pravesh Kothari.
\newblock Embedding hard learning problems into gaussian space.
\newblock In \emph{Approximation, Randomization, and Combinatorial
  Optimization. Algorithms and Techniques (APPROX/RANDOM 2014)}. Schloss
  Dagstuhl-Leibniz-Zentrum fuer Informatik, 2014.

\bibitem[Krishnamurthy et~al.(2017)Krishnamurthy, Agarwal, Huang,
  Daum{\'e}~III, and Langford]{krishnamurthy2017active}
Akshay Krishnamurthy, Alekh Agarwal, Tzu-Kuo Huang, Hal Daum{\'e}~III, and John
  Langford.
\newblock Active learning for cost-sensitive classification.
\newblock In \emph{International Conference on Machine Learning}, pages
  1915--1924. PMLR, 2017.

\bibitem[Littlestone(1987)]{littlestone1987learning}
Nick Littlestone.
\newblock Learning quickly when irrelevant attributes abound: {A} new
  linear-threshold algorithm (extended abstract).
\newblock In \emph{Proceedings of the 28th Annual Symposium on Foundations of
  Computer Science}, pages 68--77, 1987.

\bibitem[Lov{\'a}sz and Vempala(2007)]{lovasz2007geometry}
L{\'a}szl{\'o} Lov{\'a}sz and Santosh Vempala.
\newblock The geometry of logconcave functions and sampling algorithms.
\newblock \emph{Random Structures \& Algorithms}, 30\penalty0 (3):\penalty0
  307--358, 2007.

\bibitem[Massart and N{\'e}d{\'e}lec(2006)]{massart2006risk}
Pascal Massart and {\'E}lodie N{\'e}d{\'e}lec.
\newblock Risk bounds for statistical learning.
\newblock \emph{The Annals of Statistics}, pages 2326--2366, 2006.

\bibitem[Settles(2009)]{settles2009active}
Burr Settles.
\newblock Active learning literature survey.
\newblock Technical report, University of Wisconsin-Madison Department of
  Computer Sciences, 2009.

\bibitem[Tsybakov(2004)]{tsybakov2004optimal}
Alexander~B. Tsybakov.
\newblock Optimal aggregation of classifiers in statistical learning.
\newblock \emph{The Annals of Statistics}, 32\penalty0 (1):\penalty0 135--166,
  2004.

\bibitem[Valiant(1985)]{valiant1985learning}
Leslie~G. Valiant.
\newblock Learning disjunction of conjunctions.
\newblock In \emph{Proceedings of the 9th International Joint Conference on
  Artificial Intelligence}, pages 560--566, 1985.

\bibitem[Vapnik(1998)]{vapnik1998statistical}
Vladimir~Naumovich Vapnik.
\newblock \emph{Statistical Learning Theory}.
\newblock Wiley, 1998.

\bibitem[Wang and Singh(2016)]{wang2016noise}
Yining Wang and Aarti Singh.
\newblock Noise-adaptive margin-based active learning and lower bounds under
  tsybakov noise condition.
\newblock In \emph{Proceedings of the Thirtieth AAAI Conference on Artificial
  Intelligence}, pages 2180--2186, 2016.

\bibitem[Yan and Zhang(2017)]{yan2017revisiting}
Songbai Yan and Chicheng Zhang.
\newblock Revisiting perceptron: Efficient and label-optimal learning of
  halfspaces, 2017.

\bibitem[Zhang(2018)]{zhang2018efficient}
Chicheng Zhang.
\newblock Efficient active learning of sparse halfspaces.
\newblock In \emph{Proceedings of the 31st Annual Conference On Learning
  Theory}, pages 1856--1880, 2018.

\bibitem[Zhang and Chaudhuri(2014)]{zhang2014beyond}
Chicheng Zhang and Kamalika Chaudhuri.
\newblock Beyond disagreement-based agnostic active learning.
\newblock \emph{Advances in Neural Information Processing Systems},
  27:\penalty0 442--450, 2014.

\bibitem[Zhang et~al.(2020)Zhang, Shen, and Awasthi]{zhang2020efficient}
Chicheng Zhang, Jie Shen, and Pranjal Awasthi.
\newblock Efficient active learning of sparse halfspaces with arbitrary bounded
  noise.
\newblock \emph{arXiv}, pages arXiv--2002, 2020.

\bibitem[Zhang(2004)]{zhang2004statistical}
Tong Zhang.
\newblock Statistical behavior and consistency of classification methods based
  on convex risk minimization.
\newblock \emph{Annals of Statistics}, pages 56--85, 2004.

\bibitem[Zhang et~al.(2017)Zhang, Liang, and Charikar]{zhang2017hitting}
Yuchen Zhang, Percy Liang, and Moses Charikar.
\newblock A hitting time analysis of stochastic gradient langevin dynamics.
\newblock In \emph{Proceedings of the 30th Annual Conference on Learning
  Theory}, pages 1980--2022, 2017.

\end{thebibliography}

\appendix

\section{Precise settings of parameters under different noise conditions}
\label{sec:params}
In this section, we provide precise settings of bandwidth function and sample size function, and the initial target excess error function, under the three noise conditions respectively. 
\begin{enumerate}
    \item Define $b_{\MNC}(\eta, r) = \tilde{\Theta}\rbr{ \min \rbr{ r R, \frac{(1-2\eta) r R^2 L}{U \beta} } }$, $T_{\MNC}(\eta, r) = \tilde{O}\rbr{ d \rbr{ \ln\frac 1 {\delta r} }^3 \cdot \rbr{\frac{U \beta^2}{(1-2\eta) R^2 L}}^2  } $, and $\epsilon_{\MNC}(\eta) = \tilde{O}\rbr{ (1-2\eta) L R^2}$.
    
    \item Define $b_{\TNC}(A, \alpha, r) = \tilde{\Theta}\rbr{ \min \rbr{ rR, \rbr{\frac{RL}{A}}^{\frac{1-\alpha}{2\alpha-1}} \rbr{\frac{R^2 L r}{U \beta} }^{\frac \alpha {2\alpha-1}} } }$, 
    
    \[ T_{\TNC}(A, \alpha, r) = \tilde{O} \rbr{ d \rbr{ \ln\frac 1 {\delta r} }^3 \max\rbr{ \rbr{ \frac{A}{\beta RLr} }^{\frac{2-2\alpha}{2\alpha-1}} \rbr{ \frac{U \beta^2}{R^2 L}}^{\frac{2\alpha}{2\alpha-1}}, \rbr{ \frac A {R^2 L r} }^{\frac{2-2\alpha}{\alpha}} \rbr{\frac{U \beta^2}{R^2 L} }^2 } }, \]
    and $\epsilon_{\TNC}(A, \alpha) = \tilde{O}\rbr{ \frac{(LR^2)^{\frac 1\alpha}}{A^{\frac{1-\alpha} \alpha}} }$.
    
    \item Define $b_{\GTNC}(B, \alpha, r) = \tilde{\Theta}\rbr{ \min\rbr{ \frac{RL}{U\beta},1 } \cdot \min\rbr{ R r, B (R r)^{\frac1\alpha} } }$,
    \[ T_{\GTNC}(B, \alpha, r) = \tilde{O}\rbr{  d \rbr{ \ln\frac 1 {\delta r} }^3 \max\rbr{ \rbr{ \frac{\beta^2 U}{R^2 L} }^2, \rbr{ \frac{\beta^2 U}{B R L} }^2 \frac{1}{R^{\frac 2 \alpha} r^{\frac{2-2\alpha}{\alpha}}}  }  }, \]
    and $\epsilon_{\GTNC}(B, \alpha) = \tilde{O}\rbr{ B (L R^2)^{\frac 1\alpha} \rbr{ \frac 1 {U \beta} }^{\frac {1-\alpha} \alpha} }$.
\end{enumerate}

\section{Proof of Lemma~\ref{lem:pot-lb-main-text}}
\label{sec:pot}

Recall that $\pot(w) = \EE_{D_{\hat{w},b}} \sbr{ (1-2\eta(x)) \abr{\inner{w^\star}{x}} }$. We first present a more precise version of Lemma~\ref{lem:pot-lb-main-text} here.



\begin{lemma}[Restatement of Lemma~\ref{lem:pot-lb-main-text}]
Suppose $D_X$ is $(2, L, R, U, \beta)$-well behaved, and $w$ is a vector with $\tilde{\theta}(w, w^\star) =: \tilde{\theta} \geq \frac{4b}{R}$. Then:
\begin{enumerate}
    \item if $D$ satisfies $\eta$-Massart noise condition, $\pot(w) \geq \frac{(1-2\eta) R^2 L}{128 U \beta \ln\frac{2}{b U \beta}} \tilde{\theta} = \tilde{\Omega}( (1-2\eta) \tilde{\theta} )$.
    \item If $D$ satisfies $(A, \alpha)$-Tsybakov noise condition, $\pot(w) \geq     \frac{  (\frac{R b L}{8 A})^{\frac{1-\alpha}{\alpha}} R^2 L }{256 U \beta \ln\frac{2}{b U \beta}} \tilde{\theta} = \Omega( (\frac b A)^{\frac{1-\alpha}{\alpha}} \tilde{\theta} )$.
    \item If $D$ satisfies $(B, \alpha)$-geometric Tsybakov noise condition, then $\pot(w) \geq     \frac{RL}{ 16 U \beta \ln\frac{2}{b U \beta}} \cdot \min\rbr{ \frac{R \tilde{\theta} }{8}, B ( \frac{R \tilde{\theta} }{8})^{\frac{1}{\alpha}} }
    =
    \tilde{\Omega}( \min( \tilde{\theta}, B\tilde{\theta}^{\frac{1}{\alpha}}) )
    $.
\end{enumerate}
\label{lem:pot-lb}
\end{lemma}

\begin{proof}
Without loss of generality, in subsequent proof, we assume that $\theta(w, w^\star) \leq \frac{\pi}{2}$. This is because, if $\theta(w, w^\star) > \frac{\pi}{2}$, then we can consider $-w$, which satisfies that $\pot(-w) = \pot(w)$ and $\tilde{\theta}(w, w^\star) = \tilde{\theta}(-w, w^\star)$. 



Define region $R_1 := \cbr{x \in \RR^d: \inner{w}{x} \in [0, b], \inner{w^\star}{x} \in [\frac{R \sin \theta}{4}, \frac{R \sin\theta}{2}] }$. The following claim lower bounds the probability of this region.
\begin{claim}
$\PP_{x \sim D_X}(x \in R_1) 
\geq 
\frac{R b L}{4}$.
\end{claim}
\begin{proof}
We project $x$, $w$ and $w^\star$ onto the 2-dimensional subspace $V$
spanned by $\cbr{w, w^\star}$; define $\tilde{x}$, $\tilde{w}$, $\tilde{w}^\star$ $\in \RR^2$ to be the coordinates of their projections. 
Denote by $\tilde{D}_X$ the distribution of $\tilde{x}$, and denote by its probability density function $p_V$.
Without loss of generality, let $\tilde{w} = (0,1)$ and $\tilde{w}^\star = (-\sin\theta, \cos\theta)$. The projection of region $R_1$ onto this 2-d space is a parallelogram:
\[ 
\tilde{R}_1 = \cbr{\tilde{x} \in \RR^2: \inner{\tilde{w}}{\tilde{x}} \in [0, b], \inner{\tilde{w}^\star}{\tilde{x}} \in [\frac{R \sin \theta}{4}, \frac{R \sin\theta}{2}] }.
\]
The four vertices of $\tilde{R}_1$ are: 
$A = (-\frac R 4 + \frac{b}{\tan \theta}, b)$,
$B = (-\frac R 2 + \frac{b}{\tan \theta}, b)$,
$C = (-\frac{R}{4}, 0)$, $D = (-\frac{R}{2}, 0)$.
Here are some of the key properties of the parallelogram $\tilde{R}_1$.
As $\abr{OC} = \abr{CD} = \frac{R}{4}$, $\abr{BD} = \frac{b}{\sin \theta} \leq \frac{b}{\frac12 \theta} \leq \frac{R}{2}$, so any point $x$ within $\tilde{R}_1$ is at most $R$ away from the origin. 
As $D$ is $(2, L, R, U, \beta)$-well-behaved, we have that for all $\tilde{x} \in \tilde{R}_1$, $p_V(\tilde{x}) \geq L$. 
This implies that
\[
\PP_{x \sim D_X}(x \in R_1)
= \PP_{\tilde{x} \sim \tilde{D}_X }(\tilde{x} \in  \tilde{R}_1)
\geq \int_{\tilde{R}_1} p_V(\tilde{x}) d\tilde{x}
\geq 
L \cdot b \cdot \frac{R}{4}.
\qedhere
\]

\end{proof}


In addition, by the definition of $\pot$, we have
\begin{align*}
    \pot(w) 
    = & 
     \frac{\EE_{x \sim D_X} \sbr{ (1-2\eta(x)) \abr{\inner{w^\star}{x}} \ind( \abr{\inner{w}{x}} \leq b )} }{\PP_{x \sim D_X} (  \abr{\inner{w}{x} }\leq b ) } \\
    \geq & \frac{\EE_{x \sim D_X} \sbr{ (1-2\eta(x)) \abr{\inner{w^\star}{x}} \ind(x \in R_1)} }{\PP_{x \sim D_X} (  \abr{\inner{w}{x} }\leq b ) }
\end{align*}

Now we consider each noise condition separately. 
\begin{enumerate}
    \item If $D$ satisfies $\eta$-Massart noise condition, 
    then for all $x$, $1-2\eta(x) \geq 1-2\eta$; therefore,
    \[
    \EE_{x \sim D_X} \sbr{ (1-2\eta(x)) \abr{\inner{w^\star}{x}} \ind(x \in R_1)}  \geq (1-2\eta) \frac{R \sin \theta}{4} \cdot (\frac{RbL}{4})
    \]
    Thus we have 
    \[
    \pot(w) \geq \frac{(1-2\eta) \sin \theta \cdot b R^2 L}{64 b U \beta \ln\frac{2}{b U \beta}} \geq \frac{(1-2\eta) R^2 L }{128 U \beta \ln\frac{2}{b U \beta}} \theta,
    \]
    where the first inequality is from Lemma~\ref{lem:1d-prob-ub}, the second inequality is from the elementary fact that $\sin \theta \geq \frac 12 \theta$.
    
     \item Suppose $D$ satisfies $(A, \alpha)$-Tsybakov noise condition. Let $t$ satisfy $A t^{\frac{\alpha}{1-\alpha}} = \frac{Rb L}{8}$, or equivalently, 
     $t = (\frac{R b L}{8 A})^{\frac{1-\alpha}{\alpha}}$, we have:
     \begin{align*}
         \EE_{x \sim D_X} \sbr{ (1-2\eta(x)) \abr{\inner{w^\star}{x}} \ind(x \in R_1)}  &\geq 
     t \cdot \EE_{x \sim D_X} \sbr{ \ind(1-2\eta(x) \geq t) \abr{\inner{w^\star}{x}} \ind(x \in R_1)} \\
     &\geq  t \cdot \frac{R \sin \theta}{4} \cdot \EE_{x \sim D_X} \sbr{ \ind(1-2\eta(x) \geq t) \ind(x \in R_1)} \\
     &\geq  t \cdot\frac{R \sin \theta}{4} (\PP_{x \sim D_X} \sbr{\ind(x \in R_1)} - \PP_{x \sim D_X} \sbr{\ind(1-2\eta(x) \leq t)} ) \\
     &\geq  t \cdot \frac{R \sin \theta}{4} (\frac{RbL}{4} -  \frac{RbL}{8}) \\
     &=  \frac{1}{32} t \sin\theta \cdot b R^2 L.
     \end{align*}
     
     Therefore, using the elementary fact that $\sin\theta \geq \frac12 \theta$ and Lemma~\ref{lem:1d-prob-ub}, we get,
     \[
    \pot(w) \geq \frac{t \sin \theta b R^2 L}{128 b U \beta \ln\frac{2}{b U \beta}} \geq \frac{t R^2 L}{256 U \beta \ln\frac{2}{b U \beta}} \theta
    =
    \frac{  (\frac{R b L}{8 A})^{\frac{1-\alpha}{\alpha}} R^2 L }{256 U \beta \ln\frac{2}{b U \beta}} \theta.
    \]
    \item If $D$ satisfies $(B, \alpha)$-geometric Tsybakov noise condition, we have that $1-2\eta(x) \geq \min( 1, 2B |\inner{w^\star}{x}|^{\frac{1-\alpha}{\alpha}})$. Therefore,
    \begin{align*}
    &\EE_{x \sim D_X} \sbr{ (1-2\eta(x)) \abr{\inner{w^\star}{x}} \ind(x \in R_1)}  \\
    &\geq 
    \EE_{x \sim D_X} \sbr{ \abr{\inner{w^\star}{x}} \min \cbr{ 1, 2 B |\inner{w^\star}{x}|^{\frac{1-\alpha}{\alpha}}} \ind(x \in R_1)}    \\
    &\geq 
    \EE_{x \sim D_X}  \sbr{  \min\rbr{ \frac{R\sin\theta}{4}, B ( \frac{R\sin\theta}{4})^{\frac{1}{\alpha}}  }  \ind(x \in R_1)}    \\
    &\geq 
    \min\rbr{ \frac{R\theta }{8}, B ( \frac{R\theta }{8})^{\frac{1}{\alpha}} } \cdot (\frac{RbL}{4})
    \end{align*}
    where the first inequality is from the lower bound on $1-2\eta(x)$ under the $(B,\alpha)$-geometric Tsybakov noise condition; the second inequality is from the fact that for all $x$ in $R_1$, $|\inner{w^\star}{x}| \geq \frac{R \sin\theta}{4}$; the third inequality uses the claim that $\PP( x \in R_1) \geq \frac{R b L} 4$, and the fact that $\sin\theta \geq \frac\theta2$.
    
    Therefore, using Lemma~\ref{lem:1d-prob-ub},
    \[
    \pot(w) \geq 
    \frac{\min\rbr{ \frac{R\theta }{8}, B ( \frac{R\theta }{8})^{\frac{1}{\alpha}} } \cdot (\frac{RbL}{4})}{ 4 b U \beta \ln\frac{2}{b U \beta}} 
    =
    \frac{RL}{ 16 U \beta \ln\frac{8}{b U \beta}} \cdot \min\rbr{ \frac{R\theta }{8}, B ( \frac{R\theta }{8})^{\frac{1}{\alpha}} }.
    \qedhere
    \]
\end{enumerate}
\end{proof}

\section{Proof of Theorem~\ref{thm:main}}
\label{sec:proof-main}
\begin{proof}[Proof of Theorem~\ref{thm:main}]
We first show that Algorithm~\ref{alg:main} achieves PAC learning guarantee.
To this end, we show
by induction that for all $j \in \cbr{0, 1,, \ldots, k_\epsilon}$, there exists some event $E_j$ such that $\PP(E_j) \geq 1 - \frac{\delta}{2} - \sum_{l=1}^{j} \frac{\delta}{4^{l+1}}$, in which $\| v_{j+1} - w^\star \| \leq 4^{-(j+1)}$.

\paragraph{Base case.} For $j=0$,  Lemma~\ref{lem:initialize} gives that there exists some event $E_0$ that happens with probability at least $1 - \delta / 2$, in which $\| v_{i,1} - w^\star \| \leq \frac14$.

\paragraph{Inductive case.} Consider $j \geq 1$. Assume that there exists some event $E_{j-1}$ such that $\PP(E_{j-1}) \geq 1 - \frac\delta 2 - \sum_{l=1}^{j-1} \frac{\delta}{4^{l+1}}$, in which $\| v_{j} - w^\star \| \leq 4^{-j}$; conditioned on $E_{j-1}$ happening, from item~\ref{item:avg} of Lemma~\ref{lem:optimize-main}, there exists some event $F_j$ such that $\PP(F_j \mid E_{j-1}) \geq 1 - \frac{\delta}{4^{j+1}}$, under which $\| v_{j+1} - w^\star\| \leq 4^{-(j+1)}$.

Now define $E_j = F_j \cap E_{j-1}$. We have $\PP(E_j) \geq \PP(E_{j-1}) \cdot (1 - \frac{\delta}{4^{j+1}}) \geq 1 - \frac\delta 2 - \sum_{l=1}^{j} \frac{\delta}{4^{l+1}}$; in addition, on event $E_j$, we have $\| v_{j+1} - w^\star\| \leq 4^{-(j+1)}$ holding true. 

This completes the induction. We henceforth condition on event $E_{k_\epsilon}$ happening, which has probability $\geq 1 - \frac{\delta}{2} - \sum_{l=1}^{k_\epsilon} \frac{\delta}{4^{l+1}} \geq 1 - \delta$. 
In this event, the returned vector $\tilde{v} = v_{k_\epsilon+1}$ is such that 
$\| \tilde{v} - w^\star \| \leq r_\epsilon = \frac{\epsilon}{32 U \beta^2 (\ln\frac{12}{\epsilon})^2}$.
Applying item~\ref{item:angle-l2} of Lemma~\ref{lem:angle-l2}, this gives that 
$\theta(\tilde{v}, w^\star) \leq \pi r_\epsilon \leq \frac{\epsilon}{8 U \beta^2 (\ln\frac{12}{\epsilon})^2}$. Now, applying item~\ref{item:prob-angle-ub} of Lemma~\ref{lem:prob-angle} with $\gamma = \frac{\epsilon}{2}$, we have
\[
\PP_{x \sim D_X}(h_{\tilde{v}}(x) \neq h_{w^\star}(x)) 
\leq 4 U \beta^2 \rbr{\ln\frac{12}{\epsilon}}^2 \theta(\tilde{v}, w^\star) + \frac {\epsilon} 2
\leq 
\epsilon.
\]
Therefore, using triangle inequality, we get
\[
\err(h_{\tilde{v}}, D) - \err(h_{w^\star}, D) \leq 
\PP_{x \sim D_X}(h_{\tilde{v}}(x) \neq h_{w^\star}(x)) 
\leq
\epsilon.
\]

We now calculate the total label complexity of Algorithm~\ref{alg:main} under the three noise conditions respectively.
\begin{enumerate}
    \item Under the $\eta$-Massart noise condition, by Lemma~\ref{lem:initialize}, the initialization stage uses 
    $M_1 = \tilde{O}(\frac{d}{(1-2\eta)^2} )$ label queries. The total number of label queries in subsequent stages is $M_2 = \sum_{j=1}^{k_\epsilon} T_j = \tilde{O}(T_{k_{\epsilon}}) = \tilde{O}( T_{\MNC}(\eta, \epsilon) ) = \tilde{O}(\frac{d}{(1-2\eta)^2} \polylog(\frac1\epsilon) )$.
    
    Hence the total number of label queries used by Algorithm~\ref{alg:main} is 
    $M_1 + M_2 \leq \tilde{O}(\frac{d}{(1-2\eta)^2} \polylog(\frac1\epsilon) )$.
    
    \item Under the $(A, \alpha)$-Tsybakov noise condition, by Lemma~\ref{lem:initialize}, the initialization stage uses 
    \[
    M_1 
    =
    \tilde{O}\rbr{ d (1 + A^{\frac{2-2\alpha}{\alpha(2\alpha-1)}}) }
    \]
    label queries. The total number of label queries in subsequent stages is $M_2 = \sum_{j=1}^{k_\epsilon} T_j = \tilde{O}(T_{k_{\epsilon}}) = \tilde{O}( T_{\TNC}(A, \alpha, \epsilon) )$, which implies that
    \[
    M_2 = 
    \tilde{O}\rbr{ d \rbr{ \rbr{\frac{A}{\epsilon}}^{\frac{2-2\alpha}{\alpha}} + \rbr{\frac{A}{\epsilon}}^{\frac{2-2\alpha}{2\alpha-1}}} }
    =
    \tilde{O}\rbr{ d \rbr{ 1 + \rbr{\frac{A}{\epsilon}}^{\frac{2-2\alpha}{2\alpha-1}} } },
    \]
    where we use the fact that $0 \leq \frac{2-2\alpha}{\alpha} \leq \frac{2-2\alpha}{2\alpha-1}$, so $(\frac{A}{\epsilon})^{\frac{2-2\alpha}{\alpha}} \leq \max(1, (\frac{A}{\epsilon})^{\frac{2-2\alpha}{2\alpha-1}} )$ .
    
    Hence the total number of label queries by Algorithm~\ref{alg:main} is 
    \begin{align*}
    M_1 + M_2 \leq  
    \tilde{O}& \rbr{  d \rbr{ 1 + A^{\frac{2-2\alpha}{\alpha(2\alpha-1)}} + \rbr{\frac{A}{\epsilon}}^{\frac{2-2\alpha}{2\alpha-1}} }  }.
    \end{align*}
    
    \item Under the $(B,\alpha)$-geometric Tsybabov noise condition, by Lemma~\ref{lem:initialize}, the initialization stage uses
    \[
    M_1 = \tilde{O}( T_{\GTNC}(A, \alpha, r_0) ) = \tilde{O}\rbr{ d \rbr{1 + \rbr{ \frac{1}{B} }^{\frac{2}{\alpha}} } }
    \]
    label queries. 
    The total number of label queries in subsequent stages is $M_2 = \sum_{j=1}^{k_\epsilon} T_j = \tilde{O}(T_{k_{\epsilon}}) = \tilde{O}( T_{\GTNC}(A, \alpha, \epsilon) )$, which implies that
    \[
    M_2 = \tilde{O}\rbr{  d \rbr{1 + \frac{1}{B^2 }  \rbr{ \frac{1}{\epsilon} }^{\frac{2-2\alpha}{\alpha}} }  }.
    \]
    Hence the total number of label queries by Algorithm~\ref{alg:main} is
    \begin{align*}
    M_1 + M_2
    \leq \tilde{O}& \left( 
    d \rbr{ 1 + \rbr{ \frac{1}{B} }^{\frac{2}{\alpha}} + \frac{1}{B^2 } \rbr{ \frac{1}{\epsilon} }^{\frac{2-2\alpha}{\alpha}} }
    \right).
    \end{align*}
\end{enumerate}

%
This completes the label complexity upper bound proof.
\end{proof}


\section{Guarantees of \initialize: proof of Lemma~\ref{lem:initialize}}
\label{sec:initialize}

\begin{proof}[Proof of Lemma~\ref{lem:initialize}]
We first claim that with probability $1-\delta/2$, \initialize outputs $\hat{u}$ such that $\| \hat{u} - w^\star \| \leq \frac14$. This is a direct consequence of Claims~\ref{claim:stage1} and~\ref{claim:stage2} below, along with union bound.

We now calculate the total label complexity of \initialize. 
For the first stage (lines~\ref{line:stage-1-start} to~\ref{line:stage-1-end}), the total number of label queries is at most
$
L_1 = N \cdot \sum_{j=0}^{k_0} T_j
\leq 
\tilde{O}(T_{k_0})
$. 
For the second stage, the total number of label queries is at most
$
L_2 = O \rbr{ \frac{1}{\epsilon_0^2} \ln\frac{N}{\delta} }
$.
We now instantiate the label complexity bounds under the three noise conditions respectively:
\begin{enumerate}
    \item Under the $\eta$-Massart noise condition, $\epsilon_0 =\epsilon_{\MNC}(\eta) = \tilde{O}(1-2\eta)$; 
    this implies that $r_0 = \tilde{O}(\epsilon_0) = \tilde{O}(1-2\eta)$.
    Therefore,
    $L_1 \leq \tilde{O}( T_{k_0} ) = \tilde{O} ( T_{\MNC}(\eta, r_0) ) = \tilde{O}\rbr{ \frac{d}{(1-2\eta)^2} }$; in addition, $L_2 \leq \tilde{O}(\frac{1}{\epsilon_0^2}) \leq \tilde{O}\rbr{ \frac{1}{(1-2\eta)^2}}$. 
    
    Hence the total number of label queries used by \initialize is $L_1 + L_2 \leq \tilde{O}\rbr{ \frac{d}{(1-2\eta)^2} }$.
    
    \item Under the $(A, \alpha)$-Tsybakov noise condition, $\epsilon_0 = \epsilon_{\TNC}(A, \alpha) = \tilde{O}\rbr{ (\frac{ 1 }{ A })^{\frac{1-\alpha}{\alpha}} }$. This implies that $r_0 = \tilde{O}( \epsilon_0 ) = 
    \tilde{O}\rbr{ (\frac{ 1 }{ A })^{\frac{1-\alpha}{\alpha}} }$. Therefore, $L_1 \leq \tilde{O}(T_{k_0}) = \tilde{O}( T_{\TNC}(A, \alpha, r_0) ) = \tilde{O}\rbr{ d \max  ( A^{\frac{2-2\alpha}{\alpha(2\alpha-1)}}, A^{\frac{2-2\alpha}{\alpha^2}}  ) } = \tilde{O}\rbr{ d (1 + A^{\frac{2-2\alpha}{\alpha(2\alpha-1)}})}$; here we use the fact that $0 \leq \frac{2-2\alpha}{\alpha^2} \leq \frac{2-2\alpha}{\alpha(2\alpha-1)}$, hence $A^{\frac{2-2\alpha}{\alpha^2}} \leq \max(1, A^{\frac{2-2\alpha}{\alpha(2\alpha-1)}})$.
    In addition, $L_2 \leq \tilde{O}(\frac{1}{\epsilon_0^2}) \leq \tilde{O}(A^{\frac{2-2\alpha}{\alpha}})$.
    
    Hence the total number of label queries used by \initialize is 
    \[ L_1 + L_2 \leq \tilde{O}\rbr{d (1 + A^{\frac{2-2\alpha}{\alpha(2\alpha-1)}} + A^{\frac{2-2\alpha}{\alpha}} )}
    =
    \tilde{O}\rbr{ d (1 + A^{\frac{2-2\alpha}{\alpha(2\alpha-1)}} )}
    , \]
    where in the last inequality we use the fact that $0 \leq \frac{2-2\alpha}{\alpha} \leq \frac{2-2\alpha}{\alpha(2\alpha-1)}$, hence $A^{\frac{2-2\alpha}{\alpha}} \leq \max(1, A^{\frac{2-2\alpha}{\alpha(2\alpha-1)}})$.
    
    
     \item Under the $(B,\alpha)$-geometric Tsybakov noise condition, $\epsilon_0 = \epsilon_{\GTNC}(B, \alpha) = \tilde{O}( B )$. This implies that $r_0 = \tilde{O}(\epsilon_0) = \tilde{O}( B )$. Therefore, $L_1 \leq \tilde{O}(T_{k_0}) = \tilde{O}( T_{\GTNC}(A, \alpha, r_0) ) = \tilde{O}\rbr{ d \max( 1, (\frac{1}{B})^{\frac{2}{\alpha}} ) } \leq \tilde{O}\rbr{d (1 + (\frac{1}{B})^{\frac{2}{\alpha}})}$; in addition, $L_2 \leq \tilde{O}(\frac{1}{\epsilon_0^2}) = \tilde{O}( \frac{ 1 }{B^2} )$.
    
    Hence the total number of label queries by \initialize is 
    \[
    L_1 + L_2 \leq \tilde{O}\rbr{ d \rbr{ 1 + \rbr{ \frac{1}{B} }^{\frac{2}{\alpha}}
    + \frac{1}{B^2}}
    }
    \leq 
    \tilde{O}\rbr{ d \rbr{ 1 + \rbr{ \frac{1}{B} }^{\frac{2}{\alpha}} }}
    ,
    \]
    where in the last inequality, we use the fact that $0 \leq 2 \leq \frac 2 \alpha$, so $\frac{1}{B^2} \leq \max(1, (\frac{1}{B})^{\frac 2\alpha})$.
\end{enumerate}
The completes the proof of the label complexity upper bounds of \initialize.
\end{proof}





\begin{claim}
\initialize guarantees that, with probability $1-\delta/4$, there is some vector $u$ in $U$ such that $\err(h_u, D) - \err(h_{w^\star}, D) \leq \frac{\epsilon_0}{2}$. 
\label{claim:stage1}
\end{claim}
\begin{proof}
We first show that, for every trial $i$, with probability at least $\frac1{10}$, its corresponding final $v_{i,k_0+1}$ is such that $\| v_{i,k_0+1} - w^\star \| \leq r_0 = \frac{\epsilon_0}{64 U \beta^2 (\ln\frac{24}{\epsilon_0})^2}$. 
To this end, we show by induction that for all $j \in \NN$, there exists some event $E_j$ such that $\PP(E_j) \geq \frac15 - \sum_{l=1}^{j} \frac{\delta}{4^{l+1}}$, in which $\| v_{i,j+1} - w^\star \| \leq 4^{-(j+1)}$.

\paragraph{Base case.} For $j=0$, From item~\ref{item:rnd} of Lemma~\ref{lem:optimize-main}, we know that there exists some event $E_0$ that happens with probability at least $\frac15$, in which $\| v_{i,1} - w^\star \| \leq \frac14$.

\paragraph{Inductive case.} Consider $j \geq 1$. Assume that there exists some event $E_{j-1}$ such that $\PP(E_{j-1}) \geq \frac15 - \sum_{l=1}^{j-1} \frac{\delta}{4^{l+1}}$, in which $\| v_{i,j} - w^\star \| \leq 4^{-j}$; conditioned on $E_{j-1}$ happening, from item~\ref{item:avg} of Lemma~\ref{lem:optimize-main}, we know that there exists some event $F_j$ such that $\PP(F_j \mid E_{j-1}) \geq 1 - \frac{\delta}{4^{j+1}}$, in which  $\| v_{i, j+1} - w^\star\| \leq 4^{-(j+1)}$.

Now define $E_j = F_j \cap E_{j-1}$; we have $\PP(E_j) \geq \PP(E_{j-1}) \cdot (1 - \frac{\delta}{4^{j+1}}) \geq \frac15-\sum_{l=1}^j \frac{\delta}{4^{l+1}}$; in addition, on event $E_j$, we have $\| v_{i, j+1} - w^\star\| \leq 4^{-(j+1)}$ holding. 

This completes the induction. In addition, since $\delta \in (0,\frac1{10})$, we thus have shown that there exists some event $G_i$ such that $\PP(G_i) \geq \frac15 - \sum_{l=1}^{k_0} \frac{\delta}{4^{l+1}} \geq \frac 1 {10}$, in which $\| v_{i, k_0+1} - w^\star \| \leq r_0$.

As all $G_i$'s are independent, with the choice of $N = 10 \lceil \ln \frac 4 \delta \rceil$, we have that 
\[
\PP( \cup_{i=1}^N G_i ) 
=
1 - \PP( \cap_{i=1}^N \bar{G}_i )
=
1 - \prod_{i=1}^N \PP( \bar{G}_i)
\geq 
1 - \rbr{ 1 - \frac1{10} }^N
\geq 
1 - \delta/4.
\]
We henceforth condition on event $\cup_{i=1}^N G_i$ happening. In this event, there exist some $v_{i,k_0+1}$ from $U$, such that 
$\| v_{i, k_0+1} - w^\star \| \leq r_0 = \frac{\epsilon_0}{64 U \beta^2 (\ln\frac{24}{\epsilon_0})^2}$.
Applying item~\ref{item:angle-l2} of Lemma~\ref{lem:angle-l2}, this gives that for this $i$, 
$\theta(v_{i, k_0+1}, w^\star) \leq \pi r_0 \leq \frac{\epsilon_0}{16 U \beta^2 (\ln\frac{24}{\epsilon_0})^2}$. Now, applying item~\ref{item:prob-angle-ub} of Lemma~\ref{lem:prob-angle} with $\gamma = \frac{\epsilon_0}{4}$, we have
\[
\PP_{x \sim D_X}(h_{v_{i,k_0+1}}(x) \neq h_{w^\star}(x)) 
\leq 4 U \beta^2 \rbr{\ln\frac{24}{\epsilon_0}}^2 \theta(v_{i, k_0+1}, w^\star) + \frac {\epsilon_0} 4
\leq 
\frac{\epsilon_0}{2}.
\]
Finally, by triangle inequality, we get
\[
\err(h_{v_{i,k_0+1}}, D) - \err(h_{w^\star}, D) \leq 
\PP_{x \sim D_X}(h_{v_{i,k_0+1}}(x) \neq h_{w^\star}(x)) 
\leq
\frac{\epsilon_0}{2}.
\qedhere
\]
\end{proof}

\begin{claim}
Suppose there is some vector $u$ in $U$ such that $\err(h_u, D) - \err(h_{w^\star}, D) \leq \frac{\epsilon_0}{2}$; then after steps~\ref{line:stage-2-start} to~\ref{line:stage-2-end}, with probability $1-\delta/4$, the output vector $\hat{u}_0$ is such that $\| \hat{u}_0 - w^\star \| \leq \frac14$.
\label{claim:stage2}
\end{claim}

\begin{proof}
By Hoeffding's inequality and union bound, since $S$ is a set   $O(\frac{1}{\epsilon_0^2}\ln \frac {N}{\delta} )$ labeled examples drawn iid from $D$, 
we have with probability at least $1-\delta/4$, $\abr{ \err(h_{u'}, S) - \err(h_{u'}, D)} \leq \frac{\epsilon_0}{4}$ for any $u' \in U$. 
Recall that $h_{u_0}$ is the empirical 0-1 error minimizer over $\cbr{h_u: u \in U}$. We can upper bound the generalization error of $h_{u_0}$ as follows:
\begin{align*}
    \err(h_{u_0}, D) &= \err(h_u, D) + (\err(h_{u_0}, S) - \err(h_u, D)) + (\err(h_{u_0}, D)- \err(h_{u_0}, S) ) \\
    &\leq \err(h_u, D) + (\err(h_u, S) - \err(h_u, D)) + (\err(h_{u_0}, D)- \err(h_{u_0}, S) ) \\
    &\leq \rbr{\err(h_{w^\star}, D) + \frac{\epsilon_0}{2}} + \frac{\epsilon_0}{4} + \frac{\epsilon_0}{4} \\
    &= \err(h_{w^\star}, D) + \epsilon_0.
\end{align*}
We now show that, with the choice of $\epsilon_0$, under each of the three noise conditions, $\PP_{x \sim D_X}(h_{u_0}(x) \neq h_{w^\star}(x)) \leq \frac 14 L R^2 $.

\begin{enumerate}
    \item If $D$ satisfies $\eta$-Massart noise, by item~\ref{item:ex-massart} of Lemma~\ref{lem:excess-err}, 
    \[
    \frac14 (1-2\eta) L R^2
    = \epsilon_0
    \geq
    \err(h_{u_0}, D) - \err(h_{w^\star}, D) 
    \geq 
    (1-2\eta) \PP_{x \sim D_X}(h_{u_0}(x) \neq h_{w^\star}(x)).
    \]
    Thus we have $\PP_{x \sim D_X}(h_{u_0}(x) \neq h_{w^\star}(x)) \leq \frac 14 L R^2 $.
    
    \item If $D$ satisfies $(A,\alpha)$-Tsybakov noise, by item~\ref{item:ex-tsybakov} of Lemma~\ref{lem:excess-err}, 
    \[
    \tilde{O}\rbr{ \frac{(LR^2)^{\frac 1\alpha}}{A^{\frac{1-\alpha} \alpha}} }
    = \epsilon_0
    \geq
    \err(h_{u_0}, D) - \err(h_{w^\star}, D) 
    \geq 
    \frac{1}{(2A)^{\frac{1-\alpha}{\alpha}}} \PP(h_{u_0}(x) \neq h_{w^\star}(x))^{\frac{1}{\alpha}}.
    \]
    Again we have $\PP_{x \sim D_X}(h_{u_0}(x) \neq h_{w^\star}(x)) \leq \frac 14 L R^2 $.
    
    \item If $D$ is $(2, L, R, U, \beta)$-well behaved and satisfies $(B, \alpha)$-geometric Tsybakov noise, by item~\ref{item:ex-geo-tsybakov} of Lemma~\ref{lem:excess-err}, 
    
    \begin{align*}
        \tilde{O} \rbr{ B (L R^2)^{\frac 1\alpha} (\frac 1 {U \beta})^{\frac {1-\alpha} \alpha} }
        = \epsilon_0
        &\geq
        \err(h_{u_0}, D) - \err(h_{w^\star}, D) \\
        &\geq 
        B \rbr{ \frac{\PP(h_{u_0}(x) \neq h_{w^\star}(x))}{3} }^{\frac1\alpha} \cdot \rbr{ \frac{1}{12 U \beta \ln\frac{9}{\PP(h_{u_0}(x) \neq h_{w^\star}(x))}} }^{\frac{1-\alpha}{\alpha}}.
    \end{align*}
    
    Once more, we have $\PP_{x \sim D_X}(h_{u_0}(x) \neq h_{w^\star}(x)) \leq \frac 14 L R^2 $. 
    
\end{enumerate}
Finally, in conjunction with item~\ref{item:prob-angle-lb} of Lemma~\ref{lem:prob-angle} and item~\ref{item:unit-l2-angle} of Lemma~\ref{lem:angle-l2}, along with the fact that $\hat{u}_0$ is the $\ell_2$ normalization of $u_0$,
we have
\[
\frac 14 L R^2
\geq
\PP_{x \sim D_X}(h_{u_0}(x) \neq h_{w^\star}(x)) 
\geq 
L R^2 \theta(u_0, w^\star)
=
L R^2 \theta(\hat{u}_0, w^\star)
\geq L R^2 \| \hat{u}_0 - w^\star \|,
\]
which implies that $\| \hat{u}_0 - w^\star \| \leq \frac14$, concluding the proof.
\end{proof}

\section{Guarantees of \optimize}
\label{sec:optimize}

The following lemma shows that an upper bound of the average of $\pot(w_t)$'s can be used to give upper bound on the aggregated value of the $w_t$'s. It serves as an intermediate step towards proving Lemma~\ref{lem:optimize-main} in this section, which is the main performance guarantee of \optimize.

\begin{lemma}
\label{lem:agg-guarantee}
Given $r \in (0, 1]$, and a function $\pot$ such that there exists an nondecreasing function $f$, for all $w$, if $\tilde{\theta}(w,w^\star) \geq \frac{r}{2}$, then $\pot(w) \geq f(\tilde{\theta}(w,w^\star))$.
Suppose we are given a sequence of vectors $\cbr{w_t}_{t=1}^T$ such that $\frac1T\sum_{t=1}^T \pot(w_t) \leq \frac{1}{32} f(\frac{r}{2})$. Then:
\begin{enumerate}
    \item If we choose $\tau$ uniformly at random from $[T]$, choose a sign $\sigma$ uniformly at random from $\cbr{-1,+1}$,
    and define $w_{\rnd} = \sigma \hat{w}_\tau$,
    then with probability at least $\frac14$, 
    $\| w_{\rnd} - w^\star \| \leq r$.
    \label{item:rnd-agg}
    
    \item If in addition we have $r \leq \frac{1}{16}$ and for all $t$, $\| w_t - w^\star \| \leq 8 r$; 
    define $w_{\avg} = \frac1T \sum_{t=1}^T \hat{w}_t$, then we have  $\| w_{\avg} - w^\star \| \leq r$ holding deterministically. 
    \label{item:avg-agg}
\end{enumerate}
\end{lemma}
\begin{proof}
First we show a basic claim that will be used in both proofs.
\begin{claim}
Suppose $\pot(w) < f(\frac{r}{2})$, we must have $\tilde{\theta}(w,w^\star) \leq \frac{r}{2}$.
\label{claim:pot-angle}
\end{claim}
\begin{proof}
If $\tilde{\theta}(w,w^\star) > \frac{r}{2}$, by the assumption on $\pot$, it must be the case that $\pot(w) \geq f(\tilde{\theta}(w,w^\star))$; the right hand side is at least $ f(\frac{r}{2})$ as $f$ is nondecreasing. This contradicts with the premise that $\pot(w) \leq f(\frac{r}{2})$.
\end{proof}

We now prove the two items respectively.
\begin{enumerate}
\item By viewing $\frac1T\sum_{t=1}^T \pot(w_t)$ as $\EE_{\tau \sim \unif([T])} \sbr{\pot(w_\tau)}$ and using  Markov's inequality, we have that with probability at least $1-\frac{1}{32} \geq \frac12$ over the random choice of $\tau$, 
$\pot(w_\tau) < f(\frac{r}{2})$. By Claim~\ref{claim:pot-angle}, this shows that $\tilde{\theta}(w_\tau, w^\star) = \min(\theta(w_\tau, w^\star), \theta(-w_\tau, w^\star)) \leq \frac{r}{2}$ with probability at least $\frac12$.
Therefore, with $\sigma$ drawn uniformly at random from $\cbr{-1,+1}$, we can guarantee that with probability at least $\frac12 \times \frac12 = \frac14$, we have
\[
\theta(w_{\rnd}, w^\star) = \theta(\sigma w_\tau, w^\star) \leq \frac{r}{2}.
\]
In addition, as both $w_{\rnd}$ and $w^\star$ are unit vectors, item~\ref{item:unit-l2-angle} of Lemma~\ref{lem:angle-l2} gives that $\| w_{\rnd} - w^\star \| \leq \theta(w_{\rnd}, w^\star) \leq \frac r 2 \leq r$.


\item First, the additional assumptions on $r$ and $w_t$'s imply that for all $t$, $\|w_t - w^\star \| \leq 8 r \leq \frac{1}{2}$; now by item~\ref{item:angle-l2} of Lemma~\ref{lem:angle-l2}, $\theta(w_t, w^\star) \leq \frac{\pi}{2}$; consequently, for all $t$,  $\tilde{\theta}(w_t, w^\star) = \theta(w_t, w^\star)$. 

Define $S = \cbr{t \in [T]: \pot(w_t)  \geq f(\frac r 2)}$.
By viewing $\frac1T\sum_{t=1}^T \pot(w_t)$ as $\EE_{\tau \sim \unif([T])} \sbr{\pot(w_\tau)}$ and using  Markov's inequality, we have that $|S| \leq \frac{T}{32}$.

We now upper bound the values of $\| \hat{w}_t - w^\star \|$ for $t$ in $S$ and $\bar{S} := [T] \setminus S$, respectively:
\begin{enumerate}
    \item For $t$ in $S$, we have that $\| w_t - w^\star \| \leq 8 r$; using item~\ref{item:l2-normalize} of Lemma~\ref{lem:angle-l2}, we have $\| \hat{w}_t - w^\star \| \leq 16 r$.
    
    \item For $t$ in $\bar{S}$, we have $\pot(w_t) < f(\frac r 2)$; by Claim~\ref{claim:pot-angle}, this implies that $\tilde{\theta}(w_t, w^\star) \leq \frac r 2$. Therefore,
    $\theta(w_t, w^\star) = \tilde{\theta}(w_t, w^\star) \leq \frac r 2$. Using item~\ref{item:unit-l2-angle} of Lemma~\ref{lem:angle-l2}, along with the fact that $\hat{w}_t$ is a unit vector, we get $\| \hat{w}_t -  w^\star \| \leq \theta(w_t, w^\star) \leq \frac r 2$.
\end{enumerate}
Combining the two items above, and using the convexity of $\ell_2$ norm, we get
\begin{align*}
\| w_{\avg} - w^\star \|
\leq &  
\frac 1 T \sum_{t=1}^T \| \hat{w}_t  - w^\star \| \\
= & 
\frac 1 T \rbr{ \sum_{t \in S}  \| \hat{w}_t  - w^\star \| + \sum_{t \in \bar{S}} \| \hat{w}_t  - w^\star \| } \\
\leq & 
\frac {|S|} T \cdot 16 r + \rbr{1-\frac{|S|}{T}} \cdot \frac r 2 \\
\leq & r,
\end{align*}
where the second inequality uses the upper bounds on $\| w_t - w^\star \|$ for $t \in S$ and $t \in \bar{S}$, respectively; the last inequality uses the fact that $|S| \leq \frac{T}{32}$. 
\qedhere
\end{enumerate}
\end{proof}

Combining Lemmas~\ref{lem:optimize-g},~\ref{lem:pot-lb} and~\ref{lem:agg-guarantee}, we have the following important lemma that gives end-to-end guarantees of \optimize, under the three noise conditions considered in this paper, respectively.


\begin{lemma}[Generalization of Lemma~\ref{lem:optimize-main-avg-only}]
Fix $r \in (0,\frac 1 4]$, and $\delta \in (0,\frac1{10})$.
Suppose $b$ and $T$ are such that
\begin{enumerate}
    \item $b = b_{\MNC}(\eta, r), T = T_{\MNC}(\eta, r)$, if $D$ satisfies $\eta$-Massart noise condition;
    \item $b = b_{\TNC}(A, \alpha, r), T = T_{\TNC}(A, \alpha, r)$, if $D$ satisfies $(A, \alpha)$-Tsybakov noise condition with $\alpha \in (\frac12, 1]$;
    \item $b = b_{\GTNC}(B, \alpha, r), T = T_{\GTNC}(B, \alpha, r)$, if $D$ satisfies $(B, \alpha)$-geometric Tsybakov noise condition.
\end{enumerate}
Then \optimize, with input initial $v_1$ satisfying $\| v_1 - w^\star \| \leq 4 r$, feasible set $\Kcal = \cbr{w: \|w - v_1\| \leq 4 r}$, bandwidth $b$, number of iterations $T$, aggregation method $\agg$, has output $\tilde{w}$ that satisfies:
\begin{enumerate}
    \item If $\agg = \rnd$, then with probability at least $\frac15$,
    $\| \tilde{w} - w^\star \| \leq r$.
    \label{item:rnd}
    
    \item If furthermore $r \leq \frac{1}{16}$, and $\agg = \avg$, then with probability at least $1-r \delta$, 
    $\| \tilde{w} - w^\star \| \leq r$.
    \label{item:avg}
\end{enumerate}
\label{lem:optimize-main}
\end{lemma}
\begin{proof}
Define function $f_b$ to be such that
\[
f_b(\tilde{\theta})
=
\begin{cases}
\frac{(1-2\eta) R^2 L}{128 U \beta \ln\frac{8}{b U \beta}} \tilde{\theta}, & \text{$D$ satisfies $\eta$-Massart noise condition}
\\
\frac{  (\frac{R b L}{8 A})^{\frac{1-\alpha}{\alpha}} R^2 L }{256 U \beta \ln\frac{8}{b U \beta}} \tilde{\theta}, & \text{$D$ satisfies $(A, \alpha)$-Tsybakov noise condition}
\\
\frac{RL}{ 16 U \beta \ln\frac{8}{b U \beta}} \cdot \min\rbr{ \frac{R \tilde{\theta} }{8}, B ( \frac{R \tilde{\theta} }{8})^{\frac{1}{\alpha}} }, & \text{$D$ satisfies $(B, \alpha)$-geometric Tsybakov noise condition}
\end{cases}
\]
It can be checked that by the choices of  $b$ and $T$, the following three items hold simultaneously:
\begin{enumerate}
    \item $b \leq \frac{r R}{8}$, 
    \item $b \leq O\rbr{ f_b(\frac{r}{2}) }$,
    \item $T \geq \tilde{\Omega} \rbr{ d \cdot (\ln\frac 1 {\delta r})^3 \cdot (\frac{b + \beta r}{f_b(\frac{r}{2})})^2 }$.
\end{enumerate}
By Lemma~\ref{lem:optimize-g}, we have that there exists some event $E$ that happens with probability $\geq 1-\delta r$, under which there is a constant $c > 0$, such that 
\begin{align*}
\frac1T \sum_{t=1}^T \pot(w_t)
\leq & c b
+ 
c (b + \beta r) \cdot \rbr{\ln\frac{Td}{\delta b R L}} (\sqrt{\frac{d + \ln\frac1 {\delta r}}{T}} + \frac{\ln\frac1 {\delta r}}{T}) \\
\leq & \frac{1}{64} f_b\rbr{ \frac{r}{2} } + \frac{1}{64} f_b\rbr{ \frac{r}{2} } \\
= & \frac{1}{32} f_b\rbr{ \frac{r}{2} }
\end{align*}
here, the second inequality is from the fact that $b \leq O( f_b(\frac{r}{2}) )$ and $T \geq\tilde{\Omega} \rbr{ d \cdot (\ln\frac 1 {r\delta})^3 \cdot (\frac{b + \beta r}{f_b(\frac{r}{2})})^2 }$.
In addition, as $b \leq \frac{r R}{8}$, by Lemma~\ref{lem:pot-lb}, for all $w$ such that $\tilde{\theta}(w, w^\star) \geq \frac{r}{2} \geq \frac{4b}{R}$, $\pot(w) \geq f_b(\tilde{\theta}(w, w^\star))$. Hence, under event $E$, the premise of Lemma~\ref{lem:agg-guarantee} is satisfied; we now use it to conclude that:
\begin{enumerate}
    \item If $\agg = \rnd$, then by item~\ref{item:rnd-agg} of Lemma~\ref{lem:agg-guarantee}, we have that conditioned on $E$, with probability at least $\frac14$, $\tilde{w} = w_{\rnd}$ is such that $\| \tilde{w} - w^\star \| \leq r$. In summary, with probability at least $\PP(E) \cdot \frac{1}{4} \geq (1 - \delta r) \cdot \frac{1}{4}
    \geq (1 - \frac{1}{10}) \cdot \frac{1}{4} \geq \frac15$, $\| \tilde{w} - w^\star \| \leq r$.
    
    
    \item Suppose $\agg = \avg$, and $r \leq \frac{1}{16}$. Recall that in \optimize, we ensure that for all $t$, $w_t \in \Kcal$.
    In addition, from
    the definition of $\Kcal$, for all $u, v$ in $\Kcal$, $\| u - v \| \leq 8 r$. This implies that for all $t \in [T]$, $\|w_t - w^\star\| \leq 8 r$. In this case, applying item~\ref{item:avg-agg} of Lemma~\ref{lem:agg-guarantee}, we have that on event $E$, $\tilde{w} = w_{\avg}$ satisfies that
    $\| \tilde{w} - w^\star \| \leq r$. In summary, we have with probability $1-\delta r$, $\| \tilde{w} - w^\star \| \leq r$.
    \qedhere
\end{enumerate}
\end{proof}

\section{Proof of Lemma~\ref{lem:optimize-g}}
\label{sec:optimize-g}

\begin{proof}
Before we go into the proof, we set up some useful notations. Let filtration $\cbr{\Fcal_t}_{t=0}^T$ be such that for every $t \in \cbr{0,\ldots,T}$, $\Fcal_t = \sigma( w_1, x_1, y_1, \ldots, w_t, x_t, y_t, w_{t+1})$. Denote by $\EE_t$ the conditional expectation with respect to $\Fcal_t$.

First, by the definition of $\Kcal$, for all $u$, $v$ in $\Kcal$, $\| u - v \| \leq \| u - w_1 \| + \| v - w_1 \| \leq 8r$. As for every $t$, $w_t$ and $w^\star$ are both in $\Kcal$, we have $\| w_t - w^\star\| \leq 8 r$.

By standard analysis of online gradient descent~\cite[e.g.][Chapter 11]{cesa2006prediction}, with learning rate $\alpha$, constraint set $\Kcal$, regularizer $R(w) = \frac1{2} \| w - v_1\|^2$, we have that for every $u$ in $\Kcal$, 
\[
\alpha \sbr{\sum_{t=1}^T \inner{-u}{g_t} + \sum_{t=1}^T \inner{w_t}{g_t}} 
\leq 
\frac12 \| u - w_1 \|^2 - \frac12 \|u - w_{T+1} \|^2
+ \alpha^2 \sum_{t=1}^T \|g_t\|^2.
\]

Let $u = w^*$, dropping the negative term on the right hand side, we get
\begin{equation}
    \sum_{t=1}^T \inner{-w^*}{g_t} \leq \sum_{t=1}^T \inner{-w_t}{g_t} + \frac{1}{2\alpha} \|w^* - w_1\|^2+\alpha \sum_{t=1}^T \|g_t\|^2.
    \label{eqn:online-mirror-descent}
\end{equation}

Now we bound each term on the right hand side of Equation~\eqref{eqn:online-mirror-descent}. We make the following observations: 
\begin{enumerate}
    \item $\| w^\star - w_1 \|^2 \leq (4 r)^2$.
    
    \item $\abr{\inner{-w_t}{g_t}} 
    = \abr{\inner{w_t}{x_t}} 
    \leq 3\abr{\inner{\hat{w}_t}{x_t}} \leq 3 b$.
    This comes from that 
    $\| w_t \| \leq \| w^\star \| + \| w_t - w^\star\| \leq 1 + 8 r \leq 3$ and $x_t \sim D_{X \mid \hat{w}_t, b}$. Therefore, $\sum_{t=1}^T \inner{-w_t}{g_t} \leq 3 T \cdot b$.
    \item By the fact that $\|g_t\|_2 \leq \sqrt{d} \|g_t\|_\infty = \sqrt{d} \|x_t\|_\infty$ and Claim~\ref{claim:infty-norm} below, we have there exists a constant $c_1>0$, such that with probability at least $1-\delta r / 2$, 
    \[
    \sum_{t=1}^T \|g_t\|_\infty^2 \leq c_1 d T \beta^2 (\ln \frac{T d}{\delta r b R L} )^2.
    \]
\end{enumerate}

To summarize, the right hand side of Equation~\eqref{eqn:online-mirror-descent} is at most $3 T b + \frac{16 r^2}{\alpha} + c_1 \alpha d T \beta^2 (\ln \frac{T d}{\delta r b R L} )^2$; by the choice that $\alpha = \frac{r}{\beta} \sqrt{\frac 1 {d T}} / (\ln \frac{T d}{\delta r b R L} )$, we get that there exists some constant $c_2 > 0$, such that the above can be bounded by $3 T b + c_2 r \beta \sqrt{d T} (\ln \frac{T d}{\delta r b R L} )$.


Now, we lower bound the left hand side of Equation~\eqref{eqn:online-mirror-descent}. For every $t$,
$\|w^\star - w_t\|_2 \leq 8 r$; by item \ref{item:l2-normalize} of Lemma \ref{lem:angle-l2}, $\| w^\star - \hat{w}_t\| \leq 16 r$. By Lemma~\ref{lem:proj-tail}, we have \[
\PP_{x_t \sim {D_{\hat{w}_t,b}}}( \abr{ \inner{w^\star}{x_t} } \geq a )
\leq 
\exp\rbr{1 - \frac{a}{b + 16 \beta r (1 + \ln\frac{2}{R b L})} }.
\]
Note that $\abr{ \inner{w^\star}{x_t}} = \abr{ \inner{w^\star}{g_t}}$, hence \[
\PP_{x_t \sim {D_{\hat{w}_t,b}}}( \abr{ \inner{w^\star}{g_t} } \geq a )
\leq 
\exp\rbr{1 - \frac{a}{b + 16 \beta r (1 + \ln\frac{2}{R b L})} }.
\]
Applying Lemma~\ref{lem:mtg-subexp}, 
we have with probability at least $1-\delta r/2$, 
\[ \sum_{t=1}^T \EE_{t-1} \inner{-w^\star}{g_t} \leq \sum_{t=1}^T\inner{-w^\star}{g_t} + 32 \rbr{ b + 16 \beta r (1 + \ln\frac{1}{R b L})} (\sqrt{T\ln \frac{1}{\delta r}} + \ln \frac{1}{\delta r}).
\]

Putting these inequalities together, using the union bound, we have that with probability $1-\delta r$,
\begin{align}
    \sum_{t=1}^T \EE_{t-1} \inner{-w^\star}{g_t}
    \leq &
    3 T b + c_2 r \beta \sqrt{d T} (\ln \frac{T d}{\delta r b R L} )
    + 32 \rbr{ b + 8 \beta r (1 + \ln\frac{1}{R b L}) } (\sqrt{T\ln \frac{1}{\delta r}} + \ln \frac{1}{\delta r}) \nonumber \\
    \leq & 
    3 T b + c_3 (b + \beta r) (\ln \frac{T d}{\delta r b R L} ) \rbr{ \sqrt{ T (d + \ln\frac1{\delta r})} + \ln\frac1{\delta r}},
    \label{eqn:optimize-ineq}
\end{align}
for some constant $c_3 > 0$; here the second inequality is by algebra. 

Further, observe that \[
\EE_{t-1} \inner{-w^\star}{g_t}
= \EE_{(x_t, y_t) \sim D_{\hat{w}_t,b}} \sbr{ y_t\inner{w^\star}{x_t} }
= \EE_{(x_t, y_t) \sim D_{\hat{w}_t,b}} \sbr{ (1-2\eta(x_t)) \abr{\inner{w^\star}{x_t}} }
= \pot(w_t). 
\]
Dividing both sides of Equation~\eqref{eqn:optimize-ineq} by $T$ and combining with the above observation gives that, with probability $1-\delta r$,
\[
\frac1T \sum_{t=1}^T \pot(w_t)
\leq 
3 b + c_3 (b + \beta r) (\ln \frac{T d}{\delta r b R L} ) \rbr{ \sqrt{ \frac{d + \ln\frac1{\delta r}}{T}} + \frac{\ln\frac1{\delta r}}{T}}.
\qedhere
\]
\end{proof}



\begin{claim}
For any $\delta' > 0$, with probability at least $1-\delta'$, $\sum_{t=1}^T \|x_t\|_\infty^2 \leq T \beta^2 \cdot (\ln \frac{e T d}{\delta' b R L} )^2$.
\label{claim:infty-norm}
\end{claim}

\begin{proof}
By the definition of well behaved distributions, for any unit vector $w$,
$\PP(\abr{\inner{w}{x}} \geq t) \leq \exp(1-\frac{t}{\beta})$.
Applying the above inequality with $w = e_1, \ldots, e_d$ (i.e. the canonical basis vectors) and using the union bound, $\PP(\|x_t\|_\infty \geq t) \leq d\exp(1-\frac{t}{\beta})$. Therefore, 

\[
\PP_{x_t \sim D_{\hat{w}_t, b}}(\|x_t\|_\infty \geq t) 
\leq 
\frac{\PP_{x \sim D_X}(\|x_t\|_\infty \geq t)}{\PP ( \abr{\inner{\hat{w}_t}{x_t}} \leq b )} 
\leq 
\frac{d \exp(1-\frac{t}{\beta})}{b R L}
\]

By taking $t = \beta \cdot \ln \frac{e T d}{\delta' b R L}$, we have $\PP_{(x_t,y_t) \sim D_{\hat{v}_t, b}}(\|x_t\|_\infty \geq t) \leq \frac{\delta'}{T}$, i.e. with probability at least $1-\frac{\delta'}{T}$, 
\[\|x_t\|_\infty \leq \beta \cdot \ln \frac{e T d}{\delta' b R L}. \]
Applying union bound for all $t$ in $[T]$, we have with probability at least $1-\delta'$, the equation above holds for every $t$. When this event happens, $\sum_{t=1}^T \|x_t\|_\infty^2 \leq T \beta^2 (\ln \frac{e T d}{\delta' b R L} )^2$.
\end{proof}

\section{An attribute-efficient version of Algorithm \ref{alg:main}}
\label{sec:sparsity}


We now describe the changes needed for Algorithm \ref{alg:main} to achieve attribute efficiency:
\begin{enumerate}
    \item We modify \optimize (Algorithm \ref{alg:optimize}) so that:
    \begin{enumerate}
        \item at the beginning of the procedure, we define $w_1' = \HT_s(w_1)$, where $\HT_s$ is the hard-thresholding operation that zeros out all but $s$ largest entries of a vector in absolute value.
        
        \item we use a new constraint set $\Kcal = \cbr{w: \| w - w_1 \| \leq 4 r, \| w - w_1' \|_1 \leq 8 r \cdot \sqrt{2s}}$, and
        the new stepsize $\alpha = \frac{r}{\beta} \sqrt{\frac {s \ln d} { T}} / \rbr{\ln\frac{Td}{\delta r b R L}}$.
        
        \item instead of using $w_t$, we use $u_t$ to denote the iterates. 
        Set $u_1$, the initial iterate of the subsequent iterative process as $w_1'$ (as opposed to $w_1$). 
        Similarly, in subsequent aggregation processes, we aggregate over $\cbr{\hat{u}_t}$'s (as opposed to $\cbr{\hat{w}_t}$'s).
        
        \item instead of performing online gradient descent as in line \ref{line:omd}, we perform online mirror descent with regularizer $R(w) = \frac{1}{2(p-1)}\| w - u_1 \|^2$, where $p = \frac{\ln d}{\ln d - 1}$. Specifically, the update rule is: $u_{t+1} \gets \argmin_{w \in \Kcal} \rbr{ \alpha \inner{w}{g_t} + D_R(w, u_t) }$, where $D_R(u,v) = R(u) - R(v) - \inner{\nabla R(v)}{u - v}$ is the Bregman divergence induced by $R$.
    \end{enumerate}
    
    \item We modify the definitions of $T_{\MNC}$, $T_{\TNC}$ and $T_{\GTNC}$ so that $d$ is replaced with $s \ln d$. The definitions of iteration schedule $\cbr{T_j}$ are changed accordingly under the three noise conditions considered in this paper.
    
\end{enumerate}


With the above modifications to \optimize, we can show the following analogue of Lemma \ref{lem:optimize-g}. We only sketch its proof here, as it is very similar to the proof of Lemma \ref{lem:optimize-g} (see Appendix \ref{sec:optimize-g} above).
\begin{lemma}
Suppose $D_X$ is $(2, L, R, U, \beta)$-well behaved; in addition, $w^\star$, the Bayes-optimal halfspace, is $s$-sparse.
There exists a numerical constant $c > 0$ such that the following holds. \optimize, with the above modifications, with input initial vector $w_1$, target proximity $r \in (0,\frac14]$ such that $\| w_1 - w^\star \| \leq 4r$, bandwidth $b \leq \frac R 2$, number of iterations $T$, produces iterates $\cbr{u_t}_{t=1}^T$, such that
with probability $1- \delta r$, 
\[
\frac1T \sum_{t=1}^T \pot(u_t)
\leq 
c \rbr{ b
+ 
(b + \beta r) \cdot \rbr{\ln\frac{Td}{\delta r b R L}} (\sqrt{\frac{s \ln d + \ln\frac1{ \delta r}}{T}} + \frac{\ln\frac1{ \delta r}}{T}) }.
\]
\label{lem:optimize-g-sparse}
\end{lemma}
\begin{proof}[Proof sketch]
Given the premise that $\| w_1 - w^\star \| \leq 4r$, as $w_1'$ is the best $s$-sparse $\ell_2$-approximation to $w^\star$, we have $\| w_1 - w_1'\| \leq 4r$, and therefore $w_1' \in \Kcal$. 
        
In addition, by triangle inequality, $\| w^\star - w_1' \| \leq \| w_1 - w_1'\| + \| w^\star - w_1\| \leq 8r$; moreover, as both $w^\star$ and  $w_1'$ are $s$-sparse, $w^\star - w_1'$ is $2s$-sparse. The above two facts together imply that $\| w^\star - w_1' \|_1 \leq \sqrt{2s} \| w^\star - w_1' \| \leq 8r \cdot \sqrt{2s}$. This shows that $w^\star \in \Kcal$.

Denote by $q = \ln d$ the conjugate exponent of $p = \frac{\ln d}{\ln d - 1}$. By standard regret guarantees of online mirror descent \cite[e.g.][Chapter 11]{cesa2006prediction}, we have, 
\begin{equation}
    \sum_{t=1}^T \inner{-w^*}{g_t} \leq \sum_{t=1}^T \inner{-u_t}{g_t} + \frac{1}{2\alpha(p-1)} \|w^* - w_1' \|_p^2+\alpha \sum_{t=1}^T \|g_t\|_q^2.
\label{eqn:omd}
\end{equation}

We now bound the three terms on the right hand side respectively:
\begin{enumerate}
\item $\frac{1}{p-1}\| w^\star - w_1' \|_p^2 
= (\ln d -1) \| w^\star - w_1' \|_p^2
\leq \ln d \cdot \| w^\star - w_1' \|_1^2
\leq \ln d \cdot (\sqrt{2s} \| w^\star - w_1' \|)^2
\leq 2s \cdot \ln d \cdot (8 r)^2$.

\item $\abr{\inner{-u_t}{g_t}} 
    = \abr{\inner{u_t}{x_t}} 
    \leq 3 \abr{\inner{\hat{u}_t}{x_t}} \leq b$.
    This comes from that 
    $\| u_t \| \leq \| w^\star \| + \| u_t - w^\star \| \leq 1 + 8 r \leq 3$, and $x_t \sim D_{X \mid \hat{u}_t, b}$. Therefore, $\sum_{t=1}^T \inner{-u_t}{g_t} \leq 3 T \cdot b$.
    
\item For every $t$, $\|g_t\|_q 
\leq (d \|g_t\|_\infty^q)^\frac1q 
= e \|x_t\|_\infty 
\leq 3 \|x_t\|_\infty$; in addition, Claim~\ref{claim:infty-norm} implies that with probability at least $1-\delta r/2$, for all $t$,
$\sum_{t=1}^T \| x_t \|_\infty^2 \leq T \beta^2 (\ln \frac{e T d}{\delta r b R L} )^2$.
Therefore,
with probability at least $1-\delta r/2$, 
    \begin{equation*}
    \sum_{t=1}^T \|g_t\|_\infty^2 \leq c_1 T \beta^2 (\ln \frac{T d}{\delta r b R L} )^2,
    \end{equation*}
for some constant $c_1 > 0$.
\end{enumerate}

The above three items together imply that the right hand side is at most $3 T b + \frac{16 r^2 s \ln d}{\alpha} + c_1 \alpha T \beta^2 (\ln \frac{T d}{\delta r b R L} )^2$. 
With the choice of $\alpha = \frac{r}{\beta} \sqrt{\frac {s \ln d} { T}} / (\ln \frac{T d}{\delta r b R L} )$, we get that there exists some constant $c_2 > 0$, such that the above can be bounded by $3 T b + c_2 r \beta (\ln \frac{T d}{\delta b R L} ) \sqrt{T \cdot s \ln d } $.

In addition, Lemma~\ref{lem:mtg-subexp} gives that, with probability at least $1-\delta r/2$, 
\begin{equation} 
\sum_{t=1}^T \EE_{t-1} \inner{-w^\star}{g_t} \leq \sum_{t=1}^T\inner{-w^\star}{g_t} + 32 (b + 8 \beta r (1 + \ln\frac{2}{R b L})) (\sqrt{T\ln \frac{1}{\delta r}} + \ln \frac{1}{\delta r}).
\label{eqn:conc-benchmark-sparse}
\end{equation}

The lemma follows from plugging the above two bounds into Equation \eqref{eqn:omd}, applying the union bound and observing that $\EE_{t-1} \inner{-w^\star}{g_t} = \pot(w_t)$.
\end{proof}

With Lemma \ref{lem:optimize-g-sparse}, it is straightforward to show an analogue of Lemma \ref{lem:optimize-main}, and therefore an analogue of Theorem \ref{thm:main}, presented here for concreteness:
\begin{theorem}
Suppose $D$ is $(2, L, R, U, \beta)$-well behaved and satisfies one of the three noise conditions; in addition, $w^\star$, the Bayes-optimal halfspace, is $s$-sparse. 
With the settings of $\cbr{b_j}, \cbr{T_j}$, and $\epsilon_0$ under the respective noise conditions, with probability $1-\delta$, Algorithm~\ref{alg:main}, with the above modifications outputs a halfspace $\tilde{v}$, such that $\err(h_{\tilde{v}}, D) - \err(h_{w^\star}, D) \leq \epsilon$. In addition, its total number of label queries is at most:
\begin{enumerate}
    \item $\tilde{O}\rbr{ \frac{s}{(1-2\eta)^2} \polylog(d, \frac1\epsilon)}$, if $D$ satisfies $\eta$-Massart noise;
    \item $\tilde{O}\rbr{ s \polylog(d) \cdot \rbr{1 + A^{\frac{2-2\alpha}{\alpha(2\alpha-1)}} + (\frac{A}{\epsilon})^{\frac{2-2\alpha}{2\alpha-1}}} }$, if $D$ satisfies $(A, \alpha)$-Tsybakov noise with $\alpha \in (\frac12, 1]$;
    \item $\tilde{O}\rbr{ s \polylog(d) \cdot \rbr{ 1 + (\frac{1}{B})^{\frac 2 \alpha} + \frac{1}{B^2} (\frac 1 \epsilon)^{\frac{2-2\alpha}{\alpha}}} }$, if $D$ satisfies $(B,\alpha)$-geometric Tsybakov noise.
\end{enumerate}
\label{thm:main-sparse}
\end{theorem}
The proof of the above theorem is omitted, as it is almost a verbatim copy of the proof of Theorem \ref{thm:main}, with $d$ replaced by $s \ln d$.

\section{Auxiliary lemmas}
\label{sec:auxiliary}

We first provide a simple lemma showing that if a distribution is well-behaved and satisfies geometric Tsybakov noise condition, then it also approximately satisfies Tsybakov noise condition (i.e. it satisfies an analogue of the Tsybakov noise condition with extra log factors).

\begin{lemma}
    Suppose $D_X$ is $(2, L, R, U, \beta)$-well behaved and $D$ satisfies $(B, \alpha)$-geometric Tsybakov noise condition. Then for all $t \in [0,\frac12)$ and $\gamma > 0$,
    \[
    \PP\rbr{ \frac12 - \eta(x) \leq t } \leq 4 U \beta \rbr{ \frac t B }^{\frac \alpha {1-\alpha}} \ln\frac e \gamma +  \gamma.
    \]
    Furthermore,
    \[
    \PP\rbr{ \frac12 - \eta(x) \leq t } \leq 
    4 U \beta \rbr{ \frac t B }^{\frac \alpha {1-\alpha}} \ln \rbr{ \frac{2}{U \beta (\frac t B)^{\frac \alpha {1-\alpha}}} }.
    \]
    \label{lem:geomtnc-tnc}
\end{lemma}
    \begin{proof}
    We have
    $\PP(\frac12 - \eta(x) \leq t) \leq \PP( B |\inner{w^\star}{x}|^{\frac{1-\alpha}{\alpha}} \leq t ) = \PP( |\inner{w^\star}{x}| \leq (\frac{t}{B})^{\frac{\alpha}{1-\alpha}} )$; applying Lemma~\ref{lem:1d-prob-ub} below completes the proof. 
    \end{proof}

The following lemma provides conversion from excess error guarantees to guarantees on the disagreement probability with the optimal classifier $w^\star$, under the three noise conditions.

\begin{lemma}
\label{lem:excess-err}
Suppose $D_X$ is $(2, L, R, U, \beta)$-well behaved. We have the following:
\begin{enumerate}
    \item If $D$ satisfies $\eta$-Massart noise, then for any $w \in \RR^d$, $\err(h_w, D) - \err(h_{w^\star}, D) \geq (1-2\eta) \PP_{x \sim D_X}(h_w(x) \neq h_{w^\star}(x))$. 
    \label{item:ex-massart}
    \item If $D$ satisfies $(A,\alpha)$-Tsybakov noise, then for any $w \in \RR^d$, 
    $\err(h_w, D) - \err(h_{w^\star}, D) \geq \frac{1}{(2A)^{\frac{1-\alpha}{\alpha}}} \PP(h(x) \neq h_{w^\star}(x))^{\frac{1}{\alpha}}$.
    \label{item:ex-tsybakov}
    \item If $D$ is $(2, L, R, U, \beta)$-well behaved and satisfies $(B, \alpha)$-geometric Tsybakov noise, then for any $w \in \RR^d$, $\err(h_w, D) - \err(h_{w^\star}, D) \geq B (\frac{\PP(h_w(x) \neq h_{w^\star}(x))}{3})^{\frac1\alpha} \cdot (\frac{1}{12 U \beta \ln\frac{9}{\PP(h_w(x) \neq h_{w^\star}(x))}})^{\frac{1-\alpha}{\alpha}}$.
    \label{item:ex-geo-tsybakov}
    
\end{enumerate}
\end{lemma}
\begin{proof}
We prove each item respectively. In subsequent derivations, denote by $\disag = \cbr{x \in \RR^d: h_w(x) \neq h_{w^\star}(x)}$ the region of disagreement between $w$ and $w^\star$.

\begin{enumerate}
    \item It is well-known that $\err(h_w, D) - \err(h_{w^\star}, D) = \EE\sbr{ \ind\rbr{x \in \disag} (1-2\eta(x)) }$. As $\eta(x) \leq \eta$ for all $x$, we have that the right hand side is at least $(1-2\eta) \PP(x \in \disag)$, which proves the first item. 
    
    \item Similar to the last item, we have $\err(h_w, D) - \err(h_{w^\star}, D) = \EE\sbr{ \ind\rbr{x \in \disag} (1-2\eta(x)) }$. The right hand side can be lower bounded by:
    \begin{align*}
    \EE\sbr{ \ind\rbr{x \in \disag} (1-2\eta(x)) }
    \geq & 
    t \cdot \EE\sbr{ \ind\rbr{x \in \disag \wedge (1-2\eta(x)) \geq t}  } \\
    \geq & 
    t\cdot \rbr{ \PP\rbr{x \in \disag} - \PP\rbr{(1-2\eta(x) \leq t} } \\
    \geq &
    t \cdot \rbr{ \PP\rbr{x \in \disag} - A t^{\frac{\alpha}{1-\alpha}} } 
    \end{align*}
    As the above holds for any $t$, we choose $t = (\frac{\PP\rbr{x \in \disag}}{2A})^{\frac{1-\alpha}{\alpha}}$, which gives the second item. 
    

    \item Similar to the last item, we have $\EE\sbr{ \ind\rbr{x \in \disag} (1-2\eta(x)) }
    \geq 
    t \cdot \EE\sbr{ \ind\rbr{x \in \disag \wedge (1-2\eta(x)) \geq t}  }$ for all $t$.
    
    Therefore, we have that for all $t$ and $\gamma$,
    \begin{align*}
        \err(h_w, D) - \err(h_{w^\star}, D) 
        \geq & t \cdot \EE\sbr{ \ind\rbr{x \in \disag \wedge (1-2\eta(x)) \geq t}  } \\
        \geq & t \cdot ( \PP\rbr{x \in \disag} - \PP\rbr{1-2\eta(x) \leq t} ) \\
        \geq & t \cdot (\PP\rbr{x \in \disag} - 4 U \beta (\frac t B)^{\frac \alpha {1-\alpha}} \ln\frac e \gamma - \gamma)
    \end{align*}
    
   where the first two inequalities follow from the reasoning same as the previous item; the third inequality is from Lemma~\ref{lem:geomtnc-tnc}. Now, choosing $\gamma = \frac{\PP(x \in \disag)}{3}$, $t = B (\frac{\PP(x \in \disag)}{ 12 U \beta \ln\frac{9}{\PP(x \in \disag)}} )^{\frac{1-\alpha}{\alpha}}$), we get that the right hand side is at least 
   $B (\frac{\PP(x \in \disag)}{3})^{\frac1\alpha} \cdot (\frac{1}{12 U \beta \ln\frac{9}{\PP(x \in \disag)}})^{\frac{1-\alpha}{\alpha}}$, which gives the third item.
\end{enumerate}
The lemma follows.
\end{proof}

The next three lemmas provide upper and lower bounds of the probability mass of some special regions, under the assumption that $D_X$ is well-behaved. 

\begin{lemma}
If $D_X$ is $(2, L, R, U, \beta)$-well behaved, then
for any unit vector $w$ and $\gamma > 0$, we have
\[
\PP_{x \sim D_X} (  \abr{\inner{w}{x} }\leq b )
\leq 
4 b U \beta \ln \frac{e}{\gamma} + \gamma.
\]
Furthermore,
\[
\PP_{x \sim D_X} (  \abr{\inner{w}{x} }\leq b ) \leq 4 b U \beta \cdot \ln\rbr{\frac{2}{b U \beta}}. \]
Additionally, if $b \leq \frac R 2$, we have
\[
\PP_{x \sim D_X} (  \abr{\inner{w}{x} }\leq b ) 
\geq b R L.
\]
\label{lem:1d-prob-ub}
\end{lemma}

\begin{proof}
Without loss of generality, assume $w = (1,0,\ldots,0)$, then $\abr{\inner{w}{x} }\leq b$ is equivalent to $\abr{x_1} \leq b$.
For any $\gamma >0$, by the definition of well behaved distribution,  \[
\PP (\abr{x_1} \leq b) 
\leq \PP (\abr{x_1} \leq b, \abr{x_2} \leq \beta \ln \frac{e}{\gamma}) + \PP (\abr{x_2} \geq \beta \ln \frac{e}{\gamma})
\leq 4 b U \beta \ln \frac{e}{\gamma} + \gamma
\] 
Taking $\gamma = 4 b U \beta$, we have $\PP (\abr{x_1} \leq b) \leq 4 b U \beta (\ln\frac{e}{4 b U \beta} + 1) \leq  4 b U \beta \cdot  \rbr{\ln\frac{2}{b U \beta}}$.

For the last inequality, 
\[
\PP_{x \sim D_X} (  \abr{\inner{w}{x} }\leq b )
\geq 
\PP_{x \sim D_X}\rbr{ \abr{x_1} \leq b, \abr{x_2} \leq \frac{R}{2} }
\geq 
b \cdot R \cdot L. 
\qedhere
\]
\end{proof}

\begin{lemma}[\cite{diakonikolas2020learning}]
If $D$ is $(2, L, R, U, \beta)$-well behaved, then, we have for any $u, v$ in $\RR^d$, 
\begin{enumerate}
    \item $\PP_{x \sim D_X}(h_u(x) \neq h_v(x)) \geq L R^2 \theta(u,v)$.
    \label{item:prob-angle-lb}
    \item For all $\gamma > 0$, $\PP_{x \sim D_X}(h_u(x) \neq h_v(x)) \leq 4 U \beta^2 \rbr{\ln\frac{6}{\gamma}}^2 \theta(u,v) + \gamma$.
    \label{item:prob-angle-ub}
\end{enumerate} 
\label{lem:prob-angle}
\end{lemma}
\begin{proof}
Item~\ref{item:prob-angle-lb} follows directly from Claim 2.1 of~\cite{diakonikolas2020learning}.

For item~\ref{item:prob-angle-ub}, the assumption on $D_X$ implies that for any $z$ which is a 2-dimensional projection of $x$, $\PP(\| z \|_2 \geq t) \leq \PP(|z_1| \geq \frac t 2) + \PP(|z_2 | \geq \frac t 2) \leq 2\exp(1 - \frac{t}{2\beta})$. This implies that $D_X$ is $(U, R, g(\cdot))$-bounded with $g(\gamma) = 2 \beta (1 + \ln\frac{2}{\gamma}) \leq 2 \beta \ln\frac{6}{\gamma}$ in the sense of~\cite[Definition 1.2]{diakonikolas2020learning}. Applying Claim 2.1 therein, we have that for all $\gamma > 0$,
\[ \PP_{x \sim D_X}(h_u(x) \neq h_v(x)) \leq 4 U \beta^2 \rbr{\ln\frac{6}{\gamma}}^2 \theta(u,v) + \gamma. 
\qedhere
\]
\end{proof}



\begin{lemma}
If $D$ is $(2, L, R, U, \beta)$-well behaved, then, we have for any unit vectors $u$ and $v$ such that $\| u - v \|_2 \leq \rho$, any $b \leq \frac{R}{2}$, and any $a > 0$, 
\begin{equation}
\PP_{x \sim D_{u,b}}( \abr{ \inner{v}{x} } \geq a )
\leq 
\exp\rbr{1 - \frac{a}{b + \beta \rho (1 + \ln\frac{1}{R b L})} }.
\label{eqn:proj-tail}
\end{equation}
\label{lem:proj-tail}
\end{lemma}
\begin{proof}
In subsequent proof, we only focus on the case when $a \geq b + \beta \rho (1 + \ln\frac{1}{R b L})$; otherwise the lemma is trivial, as the right hand side of Equation~\eqref{eqn:proj-tail} is at least 1.

Without loss of generality, assume $u = (1,0,\ldots,0)$ and $v = (v_1, v_2, 0, \ldots, 0)$, where $v_1^2 + v_2^2 = 1$, and
$\| u - v \| = \sqrt{(v_1 - 1)^2 + v_2^2} \leq \rho$. The latter implies that $\abr{v_2} \leq \rho$.

By the definition of $D_{u,b}$,
\begin{equation}
\PP_{x \sim D_{u,b}}( \abr{ \inner{v}{x} } \geq a )
=
\frac{\PP_{x \sim D}( \abr{ v_1 x_1 + v_2 x_2  } > a, | x_1 | \leq b ) }{\PP_{x \sim D}( | x_1 | \leq b )}.
\label{eqn:cond-prob-expand}
\end{equation}

We bound the numerator and denominator respectively. For the denominator, by the assumption that $b \leq \frac R 2$, applying Lemma \ref{lem:1d-prob-ub}, we have
\[
\PP_{x \sim D_X}( | x_1 | \leq b )
\geq 
b R L.
\]

For the numerator, we upper bound it as follows:
\begin{align*}
\PP_{x \sim D}( \abr{ v_1 x_1 + v_2 x_2  } > a, | x_1 | \leq b )
& \leq 
\PP_{x \sim D}( \abr{ v_2 x_2  } > a - b ) \\
& \leq 
\PP_{x \sim D}( \abr{ x_2  } > \frac{a - b}{\rho} ) \\
& \leq 
\exp\rbr{ 1 - \frac{a - b}{ \rho \beta } }.
\end{align*}

Continuing Equation~\eqref{eqn:cond-prob-expand}, we get
\begin{equation}
\PP_{x \sim D_{u,b}}( \abr{ \inner{v}{x} } \geq a )
\leq 
\frac{1}{b R L} \exp\rbr{ 1 - \frac{a - b}{ \rho \beta } }
=
\exp\rbr{ 1 - \frac{a - b}{ \rho \beta } + \ln\frac{1}{bRL} }
.
\label{eqn:continue-expand}
\end{equation}
Now it can be easily checked that when 
$a \geq b + \beta \rho (1 + \ln\frac{1}{R b L})$, 
$1 - \frac{a - b}{ \rho \beta } + \ln\frac{1}{bRL} \leq 1 - \frac{a}{b + \beta \rho (1 + \ln\frac{1}{R b L})}$. Exponentiating both sides and combining this with Equation~\eqref{eqn:continue-expand}, we get the lemma.
\end{proof}

The following elementary lemma is useful for conversion between angle-based proximity and $\ell_2$-distance-based proximity.

\begin{lemma}[e.g.~\cite{zhang2020efficient}, Lemmas 27 and 28]
Suppose we are given two vectors $w$ and $u$, where $u$ is a unit vector. Then:
\begin{enumerate}
    \item $\| \hat{w} - u \| \leq 2 \| w - u \|$;
    \label{item:l2-normalize}
    \item $\theta(w, u) \leq \pi \| w - u \|$;
    \label{item:angle-l2}
    \item if in addition, $w$ is a unit vector, then $\|w - u \| \leq \theta(w, u)$.
    \label{item:unit-l2-angle}
\end{enumerate}
\label{lem:angle-l2}
\end{lemma}

Finally, the following basic martingale concentration inequality is used in Appendix~\ref{sec:optimize} to establish high-probability optimization guarantees of \optimize.

\begin{lemma}[e.g.~\cite{zhang2020efficient}, Lemma 36]
\label{lem:mtg-subexp}
Suppose $\cbr{Z_t}_{t=1}^T$ is a sequence of random variable adapted to the filtration $\cbr{\Fcal_t}_{t=1}^T$. Additionally, there exists some $C \geq 1$ and $\sigma > 0$, such that for every $Z_t$, and every $a > 0$, suppose $\PP(\abr{Z_t} > a \mid \Fcal_{t-1}) \leq C \exp\rbr{-\frac{a}{\sigma}}$. Then, with probability $1-\delta$,
\[
\abr{ \sum_{t=1}^T Z_t - \EE\sbr{ Z_t \mid \Fcal_{t-1} } }
\leq 
16 \sigma (\ln C + 1) \rbr{ \sqrt{T \ln\frac2\delta} + \ln\frac2\delta }.
\]
\end{lemma}

\bibliographystyle{plainnat}

\end{document}